\documentclass{article}
\usepackage{ijcai23}

\usepackage{times}
\usepackage{soul}
\usepackage{url}
\usepackage[hidelinks]{hyperref}
\usepackage[utf8]{inputenc}
\usepackage[small]{caption}
\usepackage{graphicx}
\usepackage{amsmath}
\usepackage{amsthm}
\usepackage{booktabs}
\usepackage{algorithm}
\usepackage{algorithmic}
\usepackage[switch]{lineno}

\usepackage{bm}
\usepackage{thmtools, thm-restate}
\usepackage{bbm}
\usepackage{xcolor}
\usepackage{amssymb}
\usepackage{multirow}
\usepackage{amsfonts}
\usepackage{float}
\usepackage{booktabs}
\usepackage{mathtools}
\usepackage{subfigure}
\usepackage{cuted}
\usepackage{enumitem}

\def\1{\mathbf{1}}
\def\0{\mathbf{0}}

\def\E{\mathbb{E}}

\def\x{{\bf x}}

\def\y{{\bf y}}

\def\V{\mathbb{V}}

\def\w{{\bf w}}

\def\H{\mathbb{H}}

\def\tr{\mathrm{tr}}
\def\KL{\mathrm{KL}}

\def\Cov{\mathrm{Cov}}

\def\gen{\mathrm{gen}}

\setlist[itemize]{leftmargin=*}

\newtheorem{assumption}{Assumption}
\newtheorem{remark}{Remark}
\newtheorem{definition}{Definition}
\newtheorem{lemma}{Lemma}
\newtheorem{theorem}{Theorem}

\newtheorem{proposition}{Proposition}
\newtheorem{restatetheoreminner}{Theorem}
\newenvironment{restatetheorem}[1]{%
  \renewcommand\therestatetheoreminner{#1}%
  \restatetheoreminner
}{\endrestatetheoreminner}
\newtheorem{restatepropositioninner}{Proposition}
\newenvironment{restateproposition}[1]{%
  \renewcommand\therestatepropositioninner{#1}%
  \restatepropositioninner
}{\endrestatepropositioninner}

\newcommand*\dif{\mathop{}\!\mathrm{d}}

\DeclarePairedDelimiter{\abs}{\lvert}{\rvert}
\DeclarePairedDelimiter{\norm}{\lVert}{\rVert}

\DeclarePairedDelimiter{\prn}{\lparen}{\rparen}
\DeclarePairedDelimiter{\brk}{\lbrack}{\rbrack}
\DeclarePairedDelimiter{\brc}{\lbrace}{\rbrace}
\DeclarePairedDelimiter{\ang}{\langle}{\rangle}

\allowdisplaybreaks

\title{
Understanding the Generalization Ability of Deep Learning Algorithms: A Kernelized R\'enyi's Entropy Perspective
}

\author{
    Yuxin Dong$^1$\and
    Tieliang Gong$^1$\thanks{Corresponding author.}\and
    Hong Chen$^2$\And
    Chen Li$^1$
    \affiliations
    $^1$School of Computer Science and Technology, Xi'an Jiaotong University, Xi'an 710049, China\\
    $^2$College of Science, Huazhong Agriculture University, Wuhan 430070, China\\
    \emails
    adidasgtl@gmail.com,
    dongyuxin@stu.xjtu.edu.cn,
    chenh@mail.hzau.edu.cn,
    cli@xjtu.edu.cn
}

\begin{document}

\maketitle

\begin{abstract}
     Recently, information-theoretic analysis has become a popular framework for understanding the generalization behavior of deep neural networks. It allows a direct analysis for stochastic gradient / Langevin descent (SGD/SGLD) learning algorithms without strong assumptions such as Lipschitz or convexity conditions. However, the current generalization error bounds within this framework are still far from optimal, while substantial improvements on these bounds are quite challenging due to the intractability of high-dimensional information quantities. To address this issue,  we first propose a novel information theoretical measure: kernelized R\'enyi's entropy, by utilizing operator representation in Hilbert space. It inherits the properties of Shannon's entropy and can be effectively calculated via simple random sampling, while remaining independent of the input dimension. We then establish the generalization error bounds for SGD/SGLD under kernelized R\'enyi's entropy, where the mutual information quantities can be directly calculated, enabling evaluation of the tightness of each intermediate step. We show that our information-theoretical bounds depend on the statistics of the stochastic gradients evaluated along with the iterates, and are rigorously tighter than the current state-of-the-art (SOTA) results. The theoretical findings are also supported by large-scale empirical studies\footnote{Proofs available at \url{https://github.com/Gamepiaynmo/KRE}}.
\end{abstract}

\section{Introduction}
Modern deep neural networks (DNNs) achieve astonishing success through their ability to memorize the entire training data while also generalizing well to unseen data. Generalization bounds in conventional statistical learning theory fail to explain this empirical observation since they attribute the generalization to the constrained complexity of hypothesis spaces, which are usually scale-sensitive \cite{zhang2021understanding}. Instead, recent studies discovered that the algorithmic choice has a significant influence on the generalization behavior of DNNs \cite{hardt2016train,bartlett2017spectrally}, raising broad research interests in investigating the theoretical properties of different learning algorithms \cite{pensia2018generalization,neu2021information,wang2021analyzing,li2022high}.

Stochastic gradient descent (SGD) has become the workhorse behind modern DNNs training. Despite its simplicity, SGD also enables high efficiency in complex and non-convex optimization problems \cite{bottou2018optimization}. This motivates extensive research into provable generalization bounds for deep learning algorithms. The first line of research employs the concept of uniform stability, beginning with \cite{hardt2016train} on investigating convergence in expectation and followed by enormous efforts exploiting similar ideas \cite{bassily2020stability,lei2021stability,yang2021simple,yang2021stability}. Another line of research connects the generalization of DNNs with information-theoretic analysis \cite{xu2017information}, also demonstrating great potential in analyzing noisy and iterative learning algorithms: \cite{pensia2018generalization} is the first to investigate the generalization ability of stochastic gradient Langevin dynamics (SGLD, a variant of SGD that injects Gaussian noise at each iteration), whose result is improved by following studies \cite{negrea2019information,wang2021analyzing}; \cite{neu2021information} then establishes information-theoretic bounds for SGD by introducing virtual noises through an auxiliary weight process, whose bounds are subsequently tightened in \cite{wang2021generalization}. Besides stability and information-theoretic views, researchers also provide PAC-Bayesian \cite{neyshabur2018pac,yang2019fast} and model compression \cite{arora2018stronger,zhou2018non} perspectives for generalization analysis.

Although current efforts on understanding and explaining the generalization of deep learning algorithms have yielded appealing results, these bounds are still restrictive due to their heavy reliance on strong assumptions or dimensionality of hypothesis spaces, making them easily become vacuous when applied to large-scale DNNs. For example, uniform stability-based generalization bounds usually assume Lipschitz continuity and smoothness of the empirical risk function \cite{hardt2016train,lei2021stability} or global optimum assumptions such as convexity and the Polyak-Lojasiewicz (PL) condition \cite{lei2021generalization,li2022high} to guarantee convergence, which is hard to meet in practice. On the contrary, information-theoretic generalization results do not rely on strong assumptions about the risk function, but the dimensionality of the hypothesis space, which often results in severely over-estimated upper bounds. As shown in Figure \ref{fig:tight}, there exists a $10^2$ to $10^3$ gap between the true generalization error and the SOTA information-theoretic generalization bound \cite{wang2021analyzing}. Furthermore, the intractability of high-dimensional information quantities possesses extra obstacles to further tightening these bounds, since it is impractical to evaluate their tightness compared to the actual value of intermediate information quantities used in the proof, neither numerically nor theoretically.

\begin{figure}[t]
    \centering
    \includegraphics[width=0.45\textwidth]{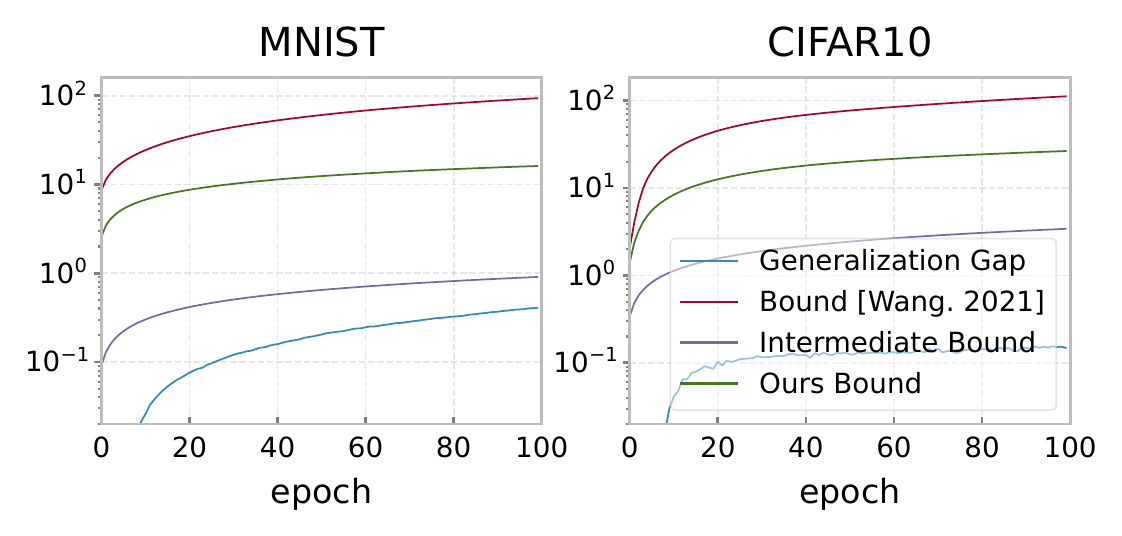}
    \caption{Generalization error of SGLD with MLP and CNN models. We provide more detailed analysis in Section \ref{sec:expr}.}
    \label{fig:tight}
\end{figure}

In this paper, we establish computable information-theoretic bounds for noisy and iterative learning algorithms by adopting an alternative information measure, namely kernelized R\'enyi's entropy. This new information quantity inherits the elegant properties of the original Shannon's definition, while being directly computable from given samples and independent of the input dimensionality. We then bound the expected generalization error of these learning algorithms with kernelized R\'enyi's entropy, where each key information quantity in the bound could be directly accessed and visualized by simple random sampling during the training process. Based on our visualization results, we improve previous information-theoretic bounds by strictly tightened ones. As an example, the above-mentioned work \cite{wang2021analyzing} upper bounds the key mutual information quantity in eq.(\ref{eq:smi_bound}) by the variance of the gradient, which is grossly over-estimated, being $10$ to $10^2$ times looser than the actual value (Intermediate Bound) as shown in Figure \ref{fig:tight}. This motivates us to reduce the gap by incorporating covariance between different dimensions of the gradient vector, which also applies to the work of \cite{pensia2018generalization,wang2021generalization}, showing significant improvement on multiple deep learning benchmarks. In summary, the key contributions of this work include:
\begin{itemize}
    \item We propose kernelized R\'enyi's entropy based on operator representation in Hilbert space. Unlike the classical Shannon's entropy, our information quantity is directly computable regardless of the dimensionality, while still being compatible with existing information-theoretic generalization frameworks.
    \item We establish mutual information generalization bounds for SGD and SGLD under the notion of kernelized R\'enyi's entropy and then visualize them on synthetic and real-world learning tasks. Our visualization results indicate multiple potential improvements in previous information-theoretic generalization results.
    \item We provide improved bounds based on one of our observations by considering correlations between different dimensions of the gradient vector. Empirical studies then demonstrate that our bounds are $5$ more times tighter compared to previous SOTA results.
\end{itemize}

\section{Preliminaries}
Given random variable $X$, we denote the corresponding sample space by $\mathcal{X}$, samples by lower-case letter $\x$, and probability distribution function (PDF) by $p_X$. We write $\norm{\cdot}$ to denote the Euclidean norm of a vector or the Frobenius norm of a matrix, and $I_d$ to denote the $d$-dimensional identity matrix.

\subsection{Problem Setting}
Let $\mathcal{Z}$ be the instance space of interest and $\mathcal{W}$ be the hypotheses space. Let $S = \{Z_i\}_{i=1}^n$ be a dataset of $n$ i.i.d. samples taking values in $\mathcal{Z}$ and $W \in \mathcal{W} \subset \mathbb{R}^d$  be the output of learning algorithm $\mathcal{A}$ according to some conditional distribution $P_{W|S}$ mapping from $\mathcal{Z}^n$ to $\mathcal{W}$. Let $\ell: \mathcal{W} \times \mathcal{Z} \rightarrow \mathbb{R}$ be a loss function. We aim to seek for a parameter $w \in \mathbb{R}^d$ that minimizes the population risk $L$, defined by
\begin{equation*}
    L(w) \triangleq \E_Z[\ell(w,Z)].
\end{equation*}
Since the data distribution is usually unknown, we turn to minimize the empirical risk
\begin{equation*}
    L_S(w) \triangleq \frac{1}{n} \sum_{i=1}^n \ell(w,Z_i).
\end{equation*}
For a learning algorithm $\mathcal{A}$ characterized by $P_{W|S}$, the corresponding generalization error is defined as the expected  difference between $L(w)$ and $L_S(w)$, i.e.  
\begin{equation*}
    \gen(W,S) \triangleq \E_{W,S} [L(W) - L_S(W)].
\end{equation*}
We assume throughout that $\ell(w,z)$ is differentiable almost everywhere with respect to $w$ for any $Z$, and $\ell(w,Z)$ is $R$-subgaussian for any $w \in \mathcal{W}$. Under these assumptions, \cite{xu2017information} shows that the generalization error of any learning algorithm $\mathcal{A}$ is bounded by
\begin{equation}
    \abs*{\gen(W,S)} \le \sqrt{\frac{2R^2 I(S; W)}{n}}, \label{eq:smi_bound}
\end{equation}
where $I(S; W)$ is the mutual information between the input dataset $S$ and the output parameter vector $W$. Due to the high-dimensional nature of modern DNNs, this quantity is generally uncomputable, possessing extra obstacles to derive tightened generalization bounds.

\subsection{R\'enyi's Entropy and Extensions}
Recall that the R\'enyi's $\alpha$-order entropy $H_\alpha(X)$ is defined on the PDF $p_X$ for a given continuous random variable $X$ in $\mathcal{X}$:
\begin{equation} \label{eq:renyi}
    H_{\alpha}(X) \triangleq \frac{1}{1 - \alpha} \log \int_{\mathcal{X}} p^\alpha(\x) \dif \x,
\end{equation}
where the limit case $\alpha \rightarrow 1$ recovers Shannon's entropy. Exactly calculating this information quantity requires knowledge about the underlying data distribution, which is usually unknown in practice. To alleviate this issue, \cite{giraldo2014measures} proposes a novel measure of entropy by utilizing the Hilbert space representation with finite data points. Specifically, it resembles quantum R\'enyi's entropy in terms of the eigenspectrum of a normalized Hermitian matrix constructed by projecting data points to a reproducing kernel Hilbert space (RKHS). In this paper, we follow this Hilbert space representation framework, with slight restrictions on the associated reproducing kernel:
\begin{assumption} \label{asmp:parzen}
    Let $\kappa(\x, \x^\prime) = \ang{\phi(\x), \phi(\x^\prime)}$ be a reproducing kernel, where $\phi$: $\mathcal{X} \mapsto \mathcal{H}$ is the corresponding feature mapping. Assume that $\kappa$ satisfies
    \begin{itemize}
        \item Normalized: $\kappa(\x, \x) = 1$ for any $\x \in \mathcal{X}$;
        \item Shift invariant: $\kappa(\x, \x^\prime) = f(\norm{\x - \x^\prime})$ for some function $f$: $\mathbb{R}^+ \mapsto \mathbb{R}^+$;
        \item $L_2$ integrable: $\forall \x \in \mathcal{X}$, $\int_\mathcal{X} \kappa^2(\x, \x^\prime) \dif \x^\prime < \infty$.
    \end{itemize}
\end{assumption}
Given random variable $X \in \mathcal{X}$, define linear operator $G_X$: $\mathcal{H} \mapsto \mathcal{H}$ as $G_X f \triangleq \E_X\brk{\phi(\x) \ang{\phi(\x), f}}$. One can verify that $\tr(G_X) = 1$ when the kernel $\kappa$ is normalized, so that the eigenvalues of $G_X$ constitute a probability distribution which is a natural density estimator for the distribution of $X$.

\section{Kernelized R\'enyi's Entropy: An Alternative Information Measure} \label{sec:kre}
In this section, we introduce kernelized R\'enyi's Entropy by extending the work of \cite{giraldo2014measures} from finite-sample cases to infinite-sample cases, enabling  direct analysis of entropy quantities based on the PDF, uninfluenced by the actual sampling process. Our definition inherits the elegant properties of the original Shannon's entropy by setting $\alpha \rightarrow 1$, while still being able to be directly accessed via simple random sampling.
\begin{proposition} \label{prop:analogue}
    Given linear operator $G_X$ defined as above on random variable $X \in \mathcal{X}$ with PDF $p_X$, we have
    \begin{multline*}
        \lim_{\alpha \rightarrow 1} \frac{1}{1-\alpha} \log \tr(G_X^\alpha) = -\tr(G_X \log G_X) \\
        = -\iint_{\mathcal{X}^2} p_X(\x) \log p_X(\x^\prime) \kappa^2(\x, \x^\prime) \dif \x \dif \x^\prime.
    \end{multline*}
\end{proposition}
Proposition \ref{prop:analogue} directly implies the following definition of kernelized R\'enyi's entropy of order $\alpha \rightarrow 1$:
\begin{definition} \label{def:ke}
    Given continuous random variable $X$ and its PDF $p_X$, the kernelized R\'enyi's entropy for $X$ of order $\alpha \rightarrow 1$ is defined as
    \begin{gather*}
        S_1(X) \triangleq -C_\kappa \iint_{\mathcal{X}^2} p_X(\x) \log p_X(\x^\prime) \kappa^2(\x, \x^\prime) \dif \x \dif \x^\prime.
    \end{gather*}
    where $C_\kappa = 1/\int_\mathcal{X} \kappa^2(0, \x) \dif \x > 0$ is the normalizing factor that let the squared kernel function integrate to $1$.
\end{definition}
Compared with the classical Shannon's definition which is intractable for high-dimensional distributions, Definition \ref{def:ke} could be directly accessed regardless of the dimensionality. To this end, one can randomly sample $m$ data points $\{\x_i\}_{i=1}^m$ from $p_X$, and denote $\hat{G}_X$: $\mathcal{H} \mapsto \mathcal{H}$ as an empirical version of  $G_X$ by $\hat{G}_X f \triangleq \frac{1}{m}\sum_{i=1}^m \phi(\x_i)\ang{\phi(\x_i),f}$. It can be verified that $\hat{G}_X$ is an unbiased estimate of $G_X$, which further implies the following finite-sample approximation to kernelized R\'enyi's entropy:
\begin{proposition} \label{prop:concentrate}
    Let $\{\x_i\}_{i=1}^m$ be i.i.d. data points sampled from $X$, and let $K \in \mathbb{R}^{m \times m}$ be the kernel matrix constructed by $K_{ij} = \frac{1}{m}\kappa(\x_i,\x_j)$. Then with confidence $1 - \delta$,
    \begin{equation}
        \abs{S_1(X) - \hat{S}_1(X)} \le \frac{9 C_\kappa \sqrt{2\log\frac{2}{\delta}}}{\sqrt[3]{m}}, \label{eq:concentrate}
    \end{equation}
    where $\hat{S}_1(X) = -C_\kappa \tr(K \log K)$.
\end{proposition}
Note that the above concentration result only involves the number of samples while remaining independent of the dimension, which allows our kernelized entropy to be directly accessed in high-dimensional cases. This property is a significant benefit in analyzing the behavior of modern DNNs, which usually involve thousands or even millions of parameters. One can also notice that Definition \ref{def:ke} can be easily extended to multivariate joint entropy by taking $\kappa = \kappa_X \otimes \kappa_Y$ as the kernel function for the joint distribution $P_{X,Y}$. With these settings, kernelized R\'enyi's divergence and mutual information can be derived accordingly:
\begin{definition} \label{def:kd}
    Given probability measures $P$, $Q$ on $\mathcal{X}$ and their PDF $p$, $q$, the kernelized R\'enyi's divergence between $P$ and $Q$ is defined as:
    \begin{gather*}
        D_1(P \parallel Q) \triangleq C_\kappa \iint_{\mathcal{X}^2} p(\x) \log \frac{p(\x^\prime)}{q(\x^\prime)} \kappa^2(\x, \x^\prime) \dif \x \dif \x^\prime.
    \end{gather*}
\end{definition}
\begin{definition} \label{def:kmi}
    Given normalized kernels $\kappa_X$, $\kappa_Y$, continuous random variables $X$, $Y$ and their PDF $p_X$, $p_Y$, the kernelized R\'enyi's mutual information between $X$ and $Y$ is defined as:
    \begin{multline*}
        I_1(X; Y) \triangleq C_{\kappa_X} C_{\kappa_Y} \iiiint_{\mathcal{Y}^2 \times \mathcal{X}^2} p_{X,Y}(\x, \y) \cdot \\
        \log \frac{p_{X,Y}(\x^\prime, \y^\prime)}{p_X(\x^\prime)p_Y(\y^\prime)} \kappa_X^2(\x, \x^\prime) \kappa_Y^2(\y, \y^\prime) \dif \x \dif \x^\prime \dif \y \dif \y^\prime.
    \end{multline*}
\end{definition}
The main difference between Shannon's entropy and kernelized R\'enyi's entropy lies in the kernel function $\kappa$. Specifically, our definition recovers the original Shannon's entropy when $\kappa$ is the Dirac-Delta function. To characterize the difference between them, we introduce the discrepancy function
\begin{equation*}
    u_X^\kappa(\x) \triangleq C_\kappa \int_\mathcal{X} \brk*{\log p_X(\x) - \log p_X(\x^\prime)} \kappa^2(\x, \x^\prime) \dif \x^\prime,
\end{equation*}
and its expected version
\begin{equation*}
    E_X^\kappa(p) \triangleq \abs*{ \int_\mathcal{X} p(\x) u_X(\x) \dif \x }.
\end{equation*}
We simply denote the \textbf{expected discrepancy}  $E_X^\kappa(p_X)$ by ${E_X^\kappa}^\prime$ and $E_X^\kappa(\hat{p}_X)$ by $E_X^\kappa$ for convenience, where $\hat{p}_X(\x) \triangleq C_\kappa \int_\mathcal{X} p_X(\x^\prime) \kappa^2(\x, \x^\prime) \dif \x^\prime$.
\begin{proposition} \label{prop:finite}
    Let $\kappa(\x, \x^\prime) = \mathbbm{1}_{\norm{\x - \x^\prime} < c}$. Assume that the PDF $p_X(\cdot)$ satisfies:
    \begin{itemize}
        \item Continuous: $\forall \x \in \mathcal{X}$, $\lim_{\x^\prime \rightarrow \x} p_X(\x^\prime) = p_X(\x)$;
        \item Positive: $\forall \x \in \mathcal{X}$, $\lim_{\x^\prime \rightarrow \x} p_X(\x^\prime) > 0$;
    \end{itemize}
    then we have $\lim_{c \rightarrow 0} E_X^\kappa \rightarrow 0$ and $\lim_{c \rightarrow 0} {E_X^\kappa}^\prime \rightarrow 0$.
\end{proposition}
Proposition \ref{prop:finite} indicates that when the normalized kernel function $C_\kappa \cdot \kappa^2$ has a very peaked bump, i.e. $c$ is small, the expected discrepancy terms $E_X^\kappa$ and ${E_X^\kappa}^\prime$ both tend to $0$. As we will show in Proposition \ref{prop:property}, setting $c \rightarrow 0$ corresponds to the case where kernelized R\'enyi's entropy recovers the original Shannon's definition. The continuity assumption above is easily satisfied when $X$ is a continuous random variable. The positiveness assumption is also naturally satisfied when $X$ is truncated between some interval $[a, b]$ so that $p_X(x) > 0$ for any $x \in [a, b]$ (e.g. randomly sampled image data are always truncated by $[0, 255]$), or the distribution of $X$ is tailed so that $p_X(\x) > 0$ for any finite $\x \in \mathcal{X}$ (e.g. the Gaussian distribution is widely used for model parameter initialization), which are common cases in modern DNNs.

\begin{proposition} \label{prop:property}
   Let $X, X^\prime \in \mathcal{X}$, $Y \in \mathcal{Y}$, $Z \in \mathcal{Z}$ be continuous random variables with probability measures $P_X$, $P_{X^\prime}$, $P_Y$ and $P_Z$ respectively. Then
    \begin{enumerate}
        \item \label{prop:ke_bound} $H(X) \le S_1(X) \le H(X) + {E_X^\kappa}^\prime$.
        \item \label{prop:kl_positive} $D_1(P_X \parallel P_{X^\prime}) \ge -E_X^\kappa$.
        \item \label{prop:mi_kl} $I_1(X; Y) = D_1(P_{X,Y} \parallel P_X \otimes P_Y) \ge 0$.
        \item \label{prop:mi_ke} $I_1(X; Y) = S_1(X) + S_1(Y) - S_1(X, Y)$.
        \item \label{prop:cond_mi} $I_1(X;Y|Z) = I_1(X;Y,Z) - I_1(X;Z)$.
        \item \label{prop:data_proc} Let $X, Y, Z$ form Markov chain $X \rightarrow Y \rightarrow Z$, then $I_1(X;Y) \ge I_1(X;Z)$ and $I_1(Y;Z) \ge I_1(X;Z)$.
    \end{enumerate}
\end{proposition}
Proposition \ref{prop:property} shows that kernelized R\'enyi's entropy inherits the essential properties of the original Shannon's entropy, thus guaranteeing compatibility with existing information theoretical analysis frameworks. Property \ref{prop:ke_bound} verifies that when $c \rightarrow 0$ in Proposition \ref{prop:finite}, kernelized R\'enyi's entropy recovers the original Shannon's entropy. Combining with the following properties, this conclusion also applies to divergence and mutual information quantities. Property \ref{prop:kl_positive} indicates that although kernelized R\'enyi's divergence is not guaranteed to be positive, it cannot be significantly less than $0$ when the kernel $\kappa$ is chosen properly. Property \ref{prop:mi_ke} and \ref{prop:cond_mi} imply that an estimate of the mutual information quantity could be acquired by estimating the value of multiple individual entropy quantities. Property \ref{prop:data_proc} is the kernelized R\'enyi's entropy version of the data processing inequality.

In the sequel, we will use Gaussian kernel in kernelized R\'enyi's information quantities to bound the expected generalization error, i.e.
\begin{equation*}
    \kappa(\x, \x^\prime) = \exp\prn{-\norm{\x - \x^\prime}_2^2/2\sigma_\kappa^2},
\end{equation*}
where $\sigma_\kappa$ is the kernel width. Note that there is a trade-off for the choice of $\sigma_\kappa$: A small $\sigma_\kappa$ implies $E_X^\kappa \approx 0$ and reduces to Shannon's entropy. However, this will cause a large normalization factor $C_\kappa$ and result in a large estimation error as shown in Proposition \ref{prop:concentrate}. In practice, we usually select $\sigma_\kappa$ according to the top 10\% to 20\% Euclidean distances between all pairwise data points as suggested by \cite{yu2019multivariate}.

\section{Generalization Bounds with Kernelized R\'enyi's Entropy} \label{sec:gen}
This section presents information-theoretic generalization bounds for iterative and noisy learning algorithms under kernelized R\'enyi's entropy. Firstly, we show that the mutual information bound for expected generalization error in eq.(\ref{eq:smi_bound}) also holds for our kernelized one:
\begin{theorem} \label{thm:gen_mi}
    Suppose that $\ell(w, Z)$ is $R$-subgaussian with respect to $Z$ for every $w \in \mathcal{W}$, then
    \begin{align*}
        \abs{\E_{S,W}[L(W) - L_S(W)]} &\le \sqrt{\frac{2R^2 \hat{I}_1(S; W)}{n}}, \quad \textrm{and} \\
        \E_{S,W}[L(W) - L_S(W)]^2 &\le \frac{4R^2(\hat{I}_1(S; W) + \log 3)}{n},
    \end{align*}
    where $\hat{I}_1(S; W) = I_1(S; W) + E_{S,W}^\kappa$.
\end{theorem}
\begin{remark}
    Theorem \ref{thm:gen_mi} provides a kernelized R\'enyi's entropy perspective for information-theoretic generalization \cite{xu2017information}. As indicated by Proposition \ref{prop:finite}, $E_{S,W}^\kappa$ is the expected discrepancy of the joint distribution between $S$ and $W$ associated with $\kappa$, which vanishes when $\sigma_\kappa \rightarrow 0$. It is worth noting that Theorem \ref{thm:gen_mi} upper bounds both the expectation and the variance of $\gen(W,S)$ by the same quantity $I_1(S;W)$, thus also yields high-probability bounds for the generalization error through concentration inequalities e.g. Markov's and Chebyshev's inequalities.
\end{remark}
Next, we apply our generalization result on mini-batched iterative and noisy learning algorithms for empirical risk minimization. Suppose algorithm $\mathcal{A}$ finishes in $T$ steps, and let $W_0 \in \mathcal{W}$ be the initial parameter vector. At the $t$-th step, a batch of data points $B_t \subset S$ independent from the current parameter vector is randomly selected and used to compute a direction for gradient descent:
\begin{equation*}
    g(w, B_t) \triangleq \frac{1}{\abs{B_t}} \sum_{z \in B_t} \nabla_w \ell(w, z).
\end{equation*}
Then the updating rule can be formalized by
\begin{equation} \label{eq:SGLD}
    W_t = W_{t-1} - \eta_t g(W_{t-1}, B_t) + \xi_t,
\end{equation}
where $W_t$ denotes the parameter vector at $t$-th step, $\eta_t$ is the learning rate and $\xi_t \in \mathcal{W}$ is a random vector independent from $W_{t-1}$ and $B_t$. Obviously, $W_0 \rightarrow W_1 \rightarrow \cdots \rightarrow W_T$ forms a Markov chain.

\subsection{Stochastic Gradient Langevin Dynamics}
The SGLD algorithm is a variant of the classical SGD algorithm by injecting random noises in each gradient update as shown in eq.(\ref{eq:SGLD}). A common choice is the isotropic Gaussian noise, i.e. $\xi_t \sim N(0, \sigma_t^2 I_d)$, since it has the maximum entropy for a fixed variance $\sigma_t^2$ which leads to the tightest upper bound. \cite{pensia2018generalization} derived the following information-theoretic generalization bound for SGLD:
\begin{lemma} \label{lm:gen_sgld}
    Let $W_T$ be the parameter vector acquired by the SGLD algorithm after $T$ updates, then
    \begin{equation}
        I(W_T;S) \le \sum_{t=1}^T \frac{d}{2}\log\prn*{\frac{\eta_t^2L}{d\sigma_t^2} + 1}, \label{eq:expr_sgld_1}
    \end{equation}
    where $L = \max_{w \in \mathcal{W}, z \in \mathcal{Z}} \norm{g(w,z)}_2^2$.
\end{lemma}
Note that the constant $L$ in Lemma \ref{lm:gen_sgld} is upper bounded by the square of the Lipschitz constant if $\ell(w, z)$ is Lipschitz continuous with regard to $w$, and the bound is dimension-dependent. This result is then improved in \cite{wang2021analyzing} by removing the Lipschitz assumption:
\begin{lemma} \label{lm:gen_sgld2}
    Under the same conditions of Lemma \ref{lm:gen_sgld}:
    \begin{equation}
        I(W_T;S) \le \sum_{t=1}^T \frac{\eta_t^2V_t}{2\sigma_t^2}, \label{eq:expr_sgld_1_2}
    \end{equation}
    where $V_t$ is the \textbf{gradient variance} at step $t$, defined by
    \begin{equation*}
        V_t \triangleq \E_{W_{t-1},B_t}\brk*{\norm{g(W_{t-1},B_t) - \E_{B_t}[g(W_{t-1},B_t)]}_2^2}.
    \end{equation*}
\end{lemma}
At first glance, the above bound does not depend on the dimensionality $d$. The gradient variance $V_t$, however, actually relies on $d$ since it is the summation of the variance raised by each dimension of the stochastic gradient vector. The main reason is that they use isotropic Gaussian distributions to upper bound the entropy of stochastic gradient, being severely over-estimated according to our empirical results (see Figure \ref{fig:regress} and \ref{fig:realdata}). To address this issue, we consider using the correlation between different dimensions of the gradient, which yields strictly tighter bounds for SGLD:
\begin{theorem} \label{thm:gen_sgld}
    Under the same conditions of Lemma \ref{lm:gen_sgld}:
    \begin{align}
        I_1(W_T;S) &\le \sum_{t=1}^T I_1\prn*{W_t;B_t|W_{t-1}} \label{eq:expr_sgld_2} \\
        &\le \sum_{t=1}^T \prn*{\frac{1}{2} \log\abs*{\frac{\eta_t^2}{\sigma_t^2}\V_t + I} + E_{W_t|W_{t-1}}^\kappa}, \label{eq:latter_step} \\
        \begin{split}
            I(W_T;S) \le \sum_{t=1}^T \frac{1}{2} \log\abs*{\frac{\eta_t^2}{\sigma_t^2}\V_t + I}, \label{eq:expr_sgld_3}
        \end{split}
    \end{align}
    where $\V_t = \Cov[g(W_{t-1}, B_t)]$ is the \textbf{gradient covariance} matrix and $\abs{\cdot}$ denotes the matrix determinant.
\end{theorem}
\begin{remark}
    Theorem \ref{thm:gen_sgld} asserts that the kernelized R\'enyi's mutual information $I_1(W_T;S)$ is upper bounded by the determinant of the gradient covariance matrix, which involves the full correlation between different dimensions of the gradient vector. The limit case $\sigma_\kappa \rightarrow 0$ implies an upper bound for Shannon's mutual information $I(W_T;S)$ in eq.(\ref{eq:expr_sgld_3}). Note that the kernelized R\'enyi's information quantities in eq.(\ref{eq:expr_sgld_2}) can be directly calculated from $B_t$, enabling us to validate the tightness of these intermediate bounds. Combining with Theorem \ref{thm:gen_mi}, one can obtain upper bounds for the expected generalization error of SGLD. 
\end{remark}
The following proposition shows that our bound is strictly tighter than that of Lemma \ref{lm:gen_sgld} and \ref{lm:gen_sgld2}.
\begin{proposition} \label{prop:tight}
    Given $\V_t$, $V_t$ and $L$ defined as above, let $\{c_i\}_{i=1}^r$ be a disjoint partition of $\{n\}$, i.e. $c_1 \cup \cdots \cup c_r = \{n\}$ and $c_i \cap c_j = \Phi$ for any $i \ne j$. Let $\V_t^i$ be the sub-matrix of $\V_t$ with columns and rows indexed by $c_i$, and define
    \begin{gather*}
        \theta_c(\V) = \frac{1}{2}\log\abs*{\frac{\eta_t^2}{\sigma_t^2}\V+I}, \quad \theta_v(V) = \frac{d}{2}\log\prn*{\frac{\eta_t^2V}{d\sigma_t^2}+1}, \\
        \textrm{then}\qquad\theta_c(\V_t) \le \sum_{i=1}^r\theta_c(\V_t^i) \le \theta_v(V_t) \le \theta_v(L).
    \end{gather*}
\end{proposition}
\begin{remark}
    The quantities $\theta_c$ and $\theta_v$ correspond to the upper bounds in Theorem \ref{thm:gen_sgld} and Lemma \ref{lm:gen_sgld2} respectively. When the model size $d$ is large, it is infeasible to calculate the entire covariance matrix $\V_t$ due to limited memory. Proposition \ref{prop:tight} suggests an alternative upper bound for $\theta_c(\V_t)$, which could be calculated using much lower memory by dividing the parameter vector into different groups according to their correlation (e.g. let each layer of the model be a group), calculating $\theta_c(\V_t^i)$ for each group and then summing them up. In the limit case where every single parameter of $W$ represents a group, calculating $\theta_c$ requires no more memory than $\theta_v$, while still being strictly tighter than the latter one.
\end{remark}

\begin{figure}[t]
    \centering
    \includegraphics[width=0.45\textwidth]{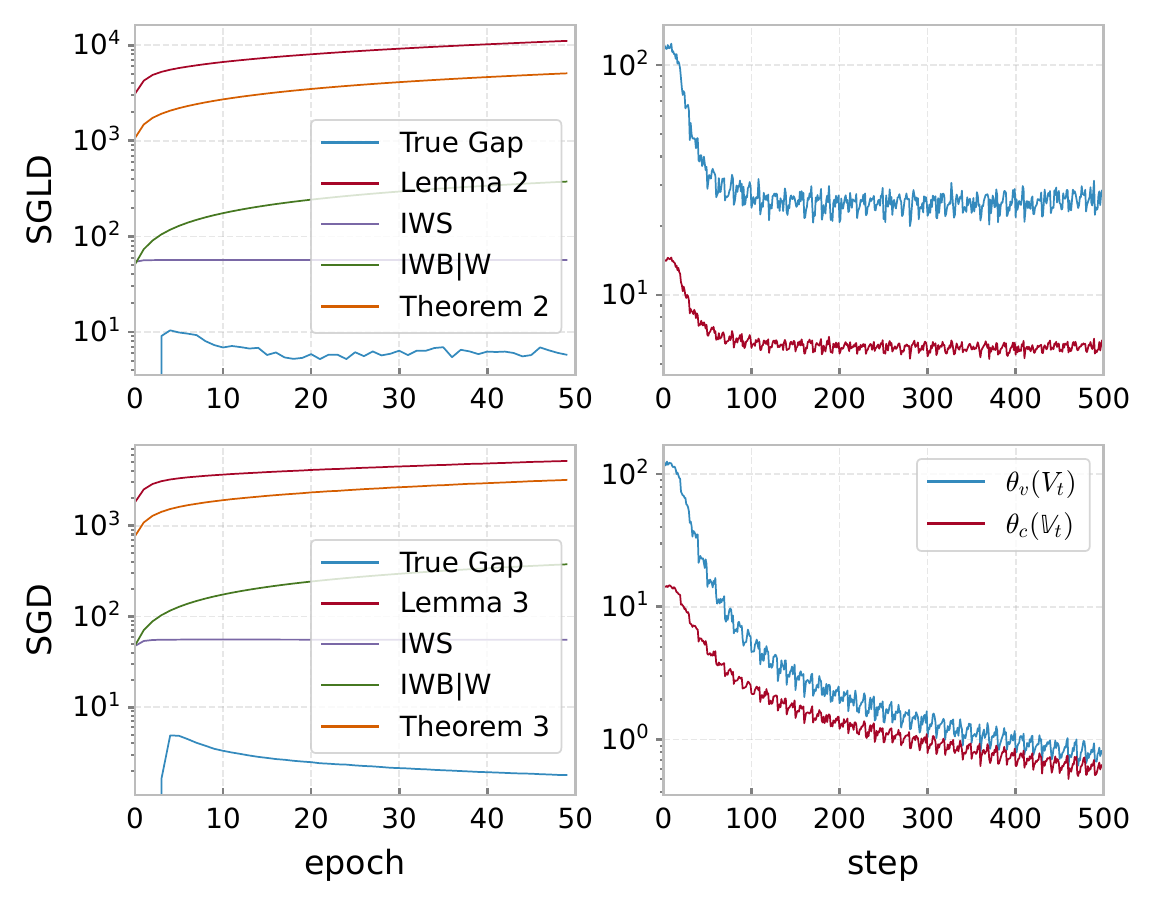}
    \caption{Comparison of generalization bounds on synthetic data.}
    \label{fig:regress}
\end{figure}

\begin{figure*}[t]
    \centering
    \includegraphics[width=0.85\textwidth]{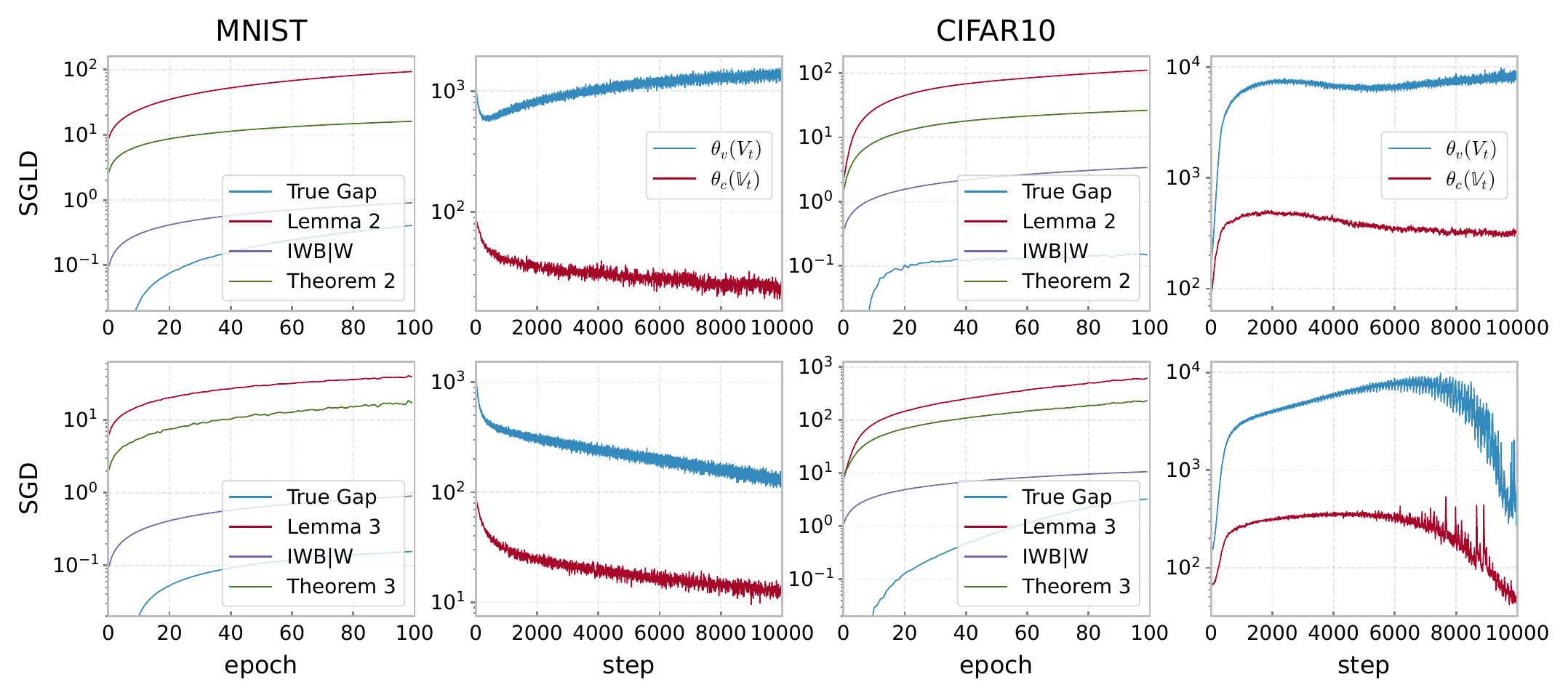}
    \caption{Visualization and comparison of information-theoretic generalization bounds for SGLD and SGD on MNIST and CIFAR10.}
    \label{fig:realdata}
\end{figure*}

\subsection{Stochastic Gradient Descent}
Unlike SGLD, the SGD algorithm does not involve random noises in each update, i.e. $\xi_t = 0$. This actually causes extra difficulty for information-theoretic generalization analysis as the strategy used to derive the SGLD bound is no longer available. To circumvent the issue, \cite{neu2021information} proposes to introduce an auxiliary weight process $\tilde{W}_t$ that manually includes virtual noises, and then bridges the differences between these two processes ($W_t$ and $\tilde{W}_t$). Let
\begin{align*}
    \tilde{W}_0 &= W_0, \quad \textrm{and} \\
    \tilde{W}_t &= \tilde{W}_{t-1} - \eta_t g(W_{t-1}, B_t) + \tilde{\xi}_t, \quad \textrm{for } t > 0,
\end{align*}
where $\tilde{\xi}_t \sim N(0, \sigma_t^2I)$ are random Gaussian vectors. Obviously, we have $\tilde{W}_t = W_t + \Delta_t$, where $\Delta_t = \sum_{i=1}^t \tilde{\xi}_i$. The recent work \cite{wang2021generalization} establishes the following information-theoretic generalization bound for SGD:
\begin{lemma} \label{lm:gen_sgd}
    Assume that $L(W_T) \le \E_{\Delta_t}[L(W_T+\Delta_t)]$ and $\ell$ is twice differentiable. Then for any $\sigma_1$, $\cdots$, $\sigma_T > 0$, we have
    \begin{gather}
        \gen(W_T;S) \le \frac{1}{2}\sum_{t=1}^T \sigma_t^2 \E_{W_T}[\H(W_T)] + \abs{\gen(\tilde{W}_T;S)}, \nonumber \\
        \textrm{and}\qquad I\prn*{\tilde{W}_T;S} \le \sum_{t=1}^T \frac{d}{2} \log\prn*{\frac{\eta_t^2V_t}{d\sigma_t^2}+1}, \label{eq:expr_sgd_1}
    \end{gather}
    where $\H(W_T) = \E_Z[\tr(H_{W_T}(Z))]$ and $H_{W_T}(Z)$ is the Hessian matrix of the loss $\ell(W_T,Z)$ with respect to $W_T$.
\end{lemma}
Again, the above mutual information bound relies on $d$. We adopt the same strategy that was explored in Theorem \ref{thm:gen_sgld} to alleviate this issue:
\begin{theorem} \label{thm:gen_sgd}
    Assume $\kappa$ satisfies Assumption \ref{asmp:parzen} and under the same conditions of Lemma \ref{lm:gen_sgd}:
    \begin{gather}
        I_1\prn*{\tilde{W}_T;S} \le \sum_{t=1}^T\prn*{\frac{1}{2}\log\abs*{\frac{\eta_t^2}{\sigma_t^2}\V_t+I} + E_{\tilde{W}_t|\tilde{W}_{t-1}}^\kappa}, \nonumber \\
        I\prn*{\tilde{W}_T;S} \le \sum_{t=1}^T \frac{1}{2}\log\abs*{\frac{\eta_t^2}{\sigma_t^2}\V_t+I}. \label{eq:expr_sgd_3}
    \end{gather}
\end{theorem}
\begin{remark}
    Theorem \ref{thm:gen_sgd} establishes information-theoretic generalization bound for the SGD algorithm within the framework of kernelized R\'enyi's entropy, where eq.(\ref{eq:expr_sgd_3}) corresponds to the limit case $\sigma_\kappa \rightarrow 0$. We highlight that our result upper bounds the generalization error by gradient covariance, which is strictly tighter than that of Lemma \ref{lm:gen_sgd} since $\theta_c(\V_t) \le \theta_v(V_t)$ as shown in Proposition \ref{prop:tight}. Note that the bounds in Theorem \ref{thm:gen_sgld} and \ref{thm:gen_sgd} could be further improved by taking into account higher order moments of the stochastic gradient, but this would impose a significant computational burden because the $s$-th moment tensor is of size $d^s$, making the potential improvement less meaningful in practice.
\end{remark}

\section{Empirical Studies} \label{sec:expr}
In this section, we visualize the computable generalization bounds for SGLD/SGD derived in the previous sections, and verify the tightness of previous results as well as our improved ones in Theorem \ref{thm:gen_sgld} and \ref{thm:gen_sgd}. For simplicity, we use constant values for learning rates $\eta_t = \eta$ and Gaussian noises $\sigma_t = \sigma$. We ignore the expected discrepancy $E_X^\kappa$ terms in computation since they tend to $0$ by taking appropriate $\sigma_\kappa$ values, and are in fact not computable. Advanced tuning techniques such as momentum, weight decay, and batch normalization are not adopted. To compute $R$ in Theorem \ref{thm:gen_mi}, we collect the loss values of each batch in each epoch and let $R = \frac{1}{2}[\max_t\ell(W_{t-1},B_t) - \min_t\ell(W_{t-1},B_t)]$. To compute $\V_t$ and $V_t$ in the mutual information upper bounds above, we use the BackPack Pytorch library \cite{dangel2020backpack} to acquire an empirical estimate of $\V_t$ and $V_t$ from each batch input of data. To compute $\H(W_T)$ in Theorem \ref{thm:gen_sgd}, we use the PyHessian library to acquire the Hessian matrix. Each experiment is repeated $100$ times to acquire i.i.d. samples of $W_t$ and $B_t$, which are then used to construct the kernel matrix $K$ in eq.(\ref{eq:concentrate}) to compute the kernelized R\'enyi's mutual information in our information-theoretic upper bounds.

\subsection{Synthetic Data}
Our first experiment incorporates a simple linear regression problem
\begin{equation*}
    y = \w^\top\x + \varepsilon,
\end{equation*}
where $\x$ is the $10$ dimensional input vector, $y$ is the regression target, $\w$ is the linear coefficient and $\varepsilon$ is some zero-mean random noise. We train an MLP with one hidden layer of width $10$. For each of the $100$ independent training processes, we generate an independent training dataset of size $n = 100$ using the same strategy. The comparison of existing theoretical generalization bounds against the true generalization gap is shown in the left side of Figure \ref{fig:regress}, in which we denote the information-theoretic bounds by the corresponding key information quantities: $I_1(W_T;S)$ by IWS, and the summation of $I_1(W_t;B_t|W_{t-1})$ by IWB$\vert$W. Note that although eq.(\ref{eq:expr_sgld_2}) is derived under the context of the SGLD algorithm, it still holds for the SGD algorithm since the only prerequisite of this inequality is the Markov chain relationship $S \rightarrow \{B_t\}_{t=1}^T \rightarrow \{W_t\}_{t=1}^T$.

As can be seen, IWS is always smaller than IWB$\vert$W, and both of them consistently fall in the interval between the curve of the True Gap and the bound of eq.(\ref{eq:expr_sgld_3}) for SGLD (or eq.(\ref{eq:expr_sgd_3}) for SGD), indicating that our approximations successfully reflect the actual behavior of these information-theoretic quantities $I(W_t;S)$ and $I(W_t;B_t|W_{t-1})$. We can gain several important insights from the visualization results: 

1) The gap between IWS and the true generalization gap indicates that the sub-gaussian constant $R$ of the loss function $\ell(w, Z)$ with respect to $Z$ is over-estimated. It is natural to assume that well-trained DNNs yield lower loss than a random initialization, and thus the constant $R$ is expected to decrease along with the training process. This observation could be adopted to further tighten the bound in eq.(\ref{eq:smi_bound}) and Theorem \ref{thm:gen_mi}. 

2) IWS quickly reaches the peak and then turns to decrease along with the training process (one can refer to the Appendix for more details), whose behavior matches that of the True Gap curve, while IWB$\vert$W keeps increasing since each $I(W_t;B_t|W_{t-1})$ is always strictly positive. This observation indicates that although the model always learns some knowledge from a new batch (i.e. $I(W_t;B_t|W_{t-1}) \ge 0$), the total information that $W$ contains about the dataset $S$ (i.e. $I(W_t;S)$) quickly reaches the upper limit: the network is actually forgetting information that learned previously. This ``forget" behavior is not captured by the current work of information-theoretic generalization bounds, resulting in a gradually increasing gap between IWS and IWB$\vert$W when the number of training epochs grows large. 

3) The remaining gap between IWB$\vert$W and our improved bounds indicates that current bounds are still far from optimal even if the full correlation of the noisy gradient is considered. This observation is supported by the recent works \cite{gurbuzbalaban2021heavy,camuto2021asymmetric}, who claim that the stochastic gradient vector generated by SGD is heavy-tailed and their entropy is significantly over-estimated by assuming Gaussian distributions. Another conjecture is that some implicit self-regularization mechanisms exist in DNNs \cite{mahoney2019traditional,martin2021implicit}, resulting in the information captured by the weights being much lower than their theoretical capacity.

The right side of Figure \ref{fig:regress} provides an intuitive comparison between the upper bounds $\theta_c$ (gradient covariance matrix) in Theorem \ref{thm:gen_sgld} and $\theta_v$ (gradient variance) in Lemma \ref{lm:gen_sgld2}. It can be seen that $\theta_v(V_t)$ is always larger than $\theta_c(\V_t)$, especially at the beginning of the training process. This observation verifies our claim in Proposition \ref{prop:tight}.

\subsection{Real-world Data}
We then visualize our computable generalization bounds on real-world datasets to demonstrate the scalability of kernelized R\'enyi's entropy. Following the experiment settings in \cite{wang2021analyzing}, we train an MLP with a wider hidden layer on MNIST and a 4-layer CNN on CIFAR10. In each of the $100$ individual training processes, a portion of data pairs is uniformly sampled from the entire dataset as the training dataset to simulate the randomness of $S$. Detailed experiment settings can be found in Appendix.

Similarly, the comparisons between different generalization bounds are reported in Figure \ref{fig:realdata}. As can be seen, the curve of IWB$\vert$W still consistently falls in the correct interval between adjacent bounds, and perfectly reflect the increasing trend of the gradient covariance upper bound. The gradient variance bound in eq.(\ref{eq:expr_sgld_1_2}) is still grossly over-estimated compared to IWB$\vert$W. For comparison, our tightened bound of eq.(\ref{eq:expr_sgld_3}) (or eq.(\ref{eq:expr_sgd_3})) covers a large portion of this gap between the curves of IWB$\vert$W and Lemma \ref{lm:gen_sgld2} for SGLD (or Lemma \ref{lm:gen_sgd} for SGD). Moreover, it can be seen that during the training process of the SGLD algorithm, the curve of $\theta_c$ quickly stops increasing and starts to decrease in the latter epochs, while the curve of $\theta_v$ is consistently increasing along with the whole training process. This observation further verifies the tightness of our improved generalization bounds.

\begin{figure}[t]
    \centering
    \includegraphics[width=0.45\textwidth]{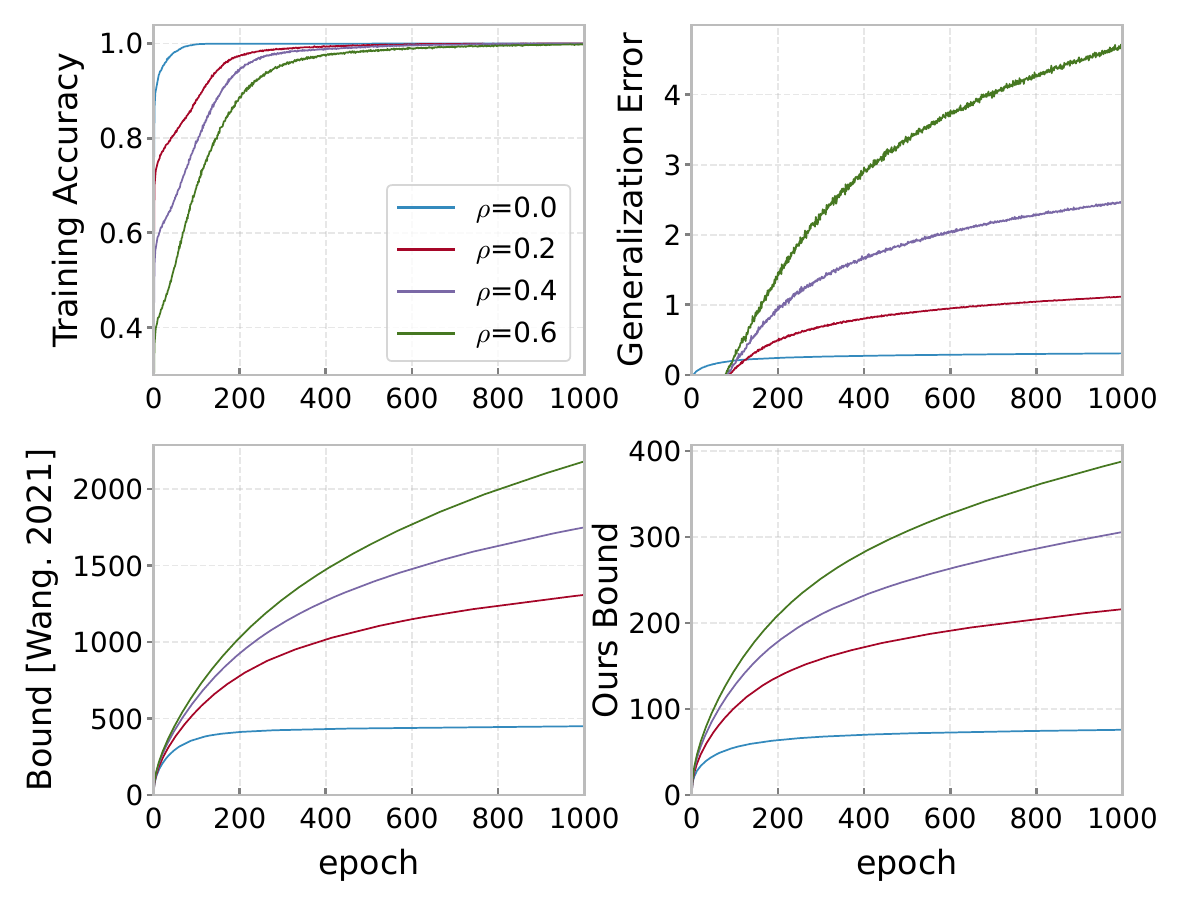}
    \caption{Random label experiment on MNIST.}
    \label{fig:label_mnist}
\end{figure}

Next, we conduct random label experiments to demonstrate the tightness of our improved generalization bounds under different levels of label noises. We keep the same experiment settings as above, while randomly replacing the training labels with noisy labels with a certain probability specified by the hyper-parameter $\rho$. As shown in Figure \ref{fig:label_mnist}, higher levels of label noise lead to higher generalization errors. While both the bound of \cite{wang2021generalization} and ours successfully reflect the trend of the true generalization gap alongside the training process, our bound is $5$ more times tighter than theirs in Lemma \ref{lm:gen_sgd}. We refer the readers to the Appendix for extra experimental results on CIFAR10 and varying model sizes. 

\section{Conclusion}
In this work, we address the common issue that Shannon's information quantities are intractable for estimation in practice. This possesses extra obstacles for information-theoretic generalization analysis, since it is impossible to evaluate the tightness of any intermediate information quantities ($I(W;S)$, $I(W_t;B_t|W_{t-1})$) used by previous generalization bounds \cite{wang2021analyzing,wang2021generalization}, resulting in the corresponding upper bounds (Lemma \ref{lm:gen_sgld2} and \ref{lm:gen_sgd}) being severely over-estimated. To address this issue, we propose an alternative measure of entropy named kernelized R\'enyi's entropy, which could be directly estimated regardless of the dimensionality, and still be compatible with existing generalization analysis frameworks. We successfully apply it to derive and visualize information-theoretic generalization bounds for noisy and iterative learning algorithms, indicating multiple potential directions for further improvement. We then prove tightened bounds for SGLD and SGD based on one of these findings, demonstrating significant improvement over existing works on multiple deep learning benchmarks.

\appendix

\section*{Acknowledgments}
This work was supported by National Key Research and Development Program of China (2021ZD0110700), National Natural Science Foundation of China (62106191, 12071166, 62192781, 61721002), the Research Council of Norway (RCN) under grant 309439, Innovation Research Team of Ministry of Education (IRT\_17R86), Project of China Knowledge Centre for Engineering Science and Technology and Project of Chinese Academy of Engineering (The Online and Offline Mixed Educational Service System for The Belt and Road Training in MOOC China).

\bibliographystyle{named}
\bibliography{references}

\newpage
\onecolumn

\begin{center}
    \LARGE \bf
    Appendix of ``Understanding the Generalization Ability of Deep Learning Algorithms: A Kernelized R\'enyi's Entropy Perspective''
\end{center}
\vspace{10pt}

\section{Proof of Section \ref{sec:kre}}

\subsection{Proof of Proposition \ref{prop:analogue}}
\begin{restateproposition}{\ref{prop:analogue}}[Restate]
    Given linear operator $G_X$ defined as above on random variable $X \in \mathcal{X}$ with PDF $p_X$, we have
    \begin{equation*}
        \lim_{\alpha \rightarrow 1} \frac{1}{1-\alpha} \log \tr(G_X^\alpha) = -\tr(G_X \log G_X) = -\iint_{\mathcal{X}^2} p_X(\x) \log p_X(\x^\prime) \kappa^2(\x, \x^\prime) \dif \x \dif \x^\prime.
    \end{equation*}
\end{restateproposition}
\begin{proof}
    Let $\{\psi_i\}_{i=1}^{N_\mathcal{H}}$ be a complete orthogonal basis for $\mathcal{H}$. Then by L'Hopital's rule:
    \begin{align*}
        \lim_{\alpha \rightarrow 1} \frac{1}{1-\alpha} \log \tr(G_X^\alpha) &= \lim_{\alpha \rightarrow 1} -\frac{\frac{\partial}{\partial \alpha} \tr(G_X^\alpha)}{\tr(G_X^\alpha)} = -\tr(G_X \log G_X) \\
        &= -\sum_{i=1}^{N_\mathcal{H}} \ang{G_X \psi_i, \log G_X \psi_i} \\
        &= -\sum_{i=1}^{N_\mathcal{H}} \int_\mathcal{X} p_X(\x) \ang{\psi_i, \phi(\x)} \ang{\phi(\x), \log G_X \psi_i} \dif \x \\
        &= -\int_\mathcal{X} p_X(\x) \ang{\log G_X \phi(\x), \sum_{i=1}^{N_\mathcal{H}} \ang{\psi_i, \phi(\x)} \psi_i} \dif \x \\
        &= -\int_\mathcal{X} p_X(\x) \ang{\log G_X \phi(\x), \phi(\x)} \dif \x \\
        &= -\int_\mathcal{X} p_X(\x) \ang{\int_\mathcal{X} \log p_X(\x^\prime) \phi(\x^\prime) \ang{\phi(\x^\prime), \phi(\x)} \dif \x^\prime, \phi(\x)} \dif \x \\
        &= -\iint_{\mathcal{X}^2} p_X(\x) \log p_X(\x^\prime) \ang{\phi(\x^\prime), \phi(\x)} \ang{\phi(\x^\prime), \phi(\x)} \dif \x \dif \x^\prime \\
        &= -\iint_{\mathcal{X}^2} p_X(\x) \log p_X(\x^\prime) \kappa^2(\x, \x^\prime) \dif \x \dif \x^\prime.
    \end{align*}
\end{proof}

\subsection{Proof of Proposition \ref{prop:concentrate}}
\begin{restateproposition}{\ref{prop:concentrate}}[Restate]
    Let $\{\x_i\}_{i=1}^m$ be i.i.d. data points sampled from $X$, and let $K \in \mathbb{R}^{m \times m}$ be the kernel matrix constructed by $K_{ij} = \frac{1}{m}\kappa(\x_i,\x_j)$. Then with confidence $1 - \delta$,
    \begin{equation*}
        \abs{S_1(X) - \hat{S}_1(X)} \le \frac{9 C_\kappa \sqrt{2\log\frac{2}{\delta}}}{\sqrt[3]{m}},
    \end{equation*}
    where $\hat{S}_1(X) = -C_\kappa \tr(K \log K)$.
\end{restateproposition}
\begin{proof}
    Let $\lambda_i$ and $\mu_i$, $i \in [1,m]$ be the eigenvalues of $G_X$ and $\hat{G}_X$ respectively. Following the proof of Theorem 6.2 in \cite{giraldo2014measures} while taking $\varphi(x) = \abs{x}$, we have that with probability $1 - \delta$,
    \begin{equation*}
        \sum_{i=1}^m \abs{\lambda_i - \mu_i} \le C \sqrt{\frac{2\log\frac{2}{\delta}}{m}},
    \end{equation*}
    where $C = \max_{\x \in \mathcal{X}} \kappa(\x, \x) = 1$. Then for any $s > 0$, we have
    \begin{align}
        \abs*{\tr(G_X \log G_X) - \tr(\hat{G}_X \log \hat{G}_X)} &= \abs*{\sum_{i=1}^m \lambda_i \log\lambda_i - \sum_{i=1}^m \mu_i \log\mu_i} \le \sum_{i=1}^m \abs{\lambda_i \log\lambda_i - \mu_i \log\mu_i} \nonumber \\
        &\le \sum_{i=1}^m \max\prn*{-\abs{\lambda_i - \mu_i}\log\abs{\lambda_i - \mu_i}, -(1-\abs{\lambda_i - \mu_i})\log(1-\abs{\lambda_i - \mu_i})} \label{eq:concentrate1} \\
        &\le \sum_{i=1}^m -\sqrt{\frac{2\log\frac{2}{\delta}}{m^3}} \log \sqrt{\frac{2\log\frac{2}{\delta}}{m^3}} = \sqrt{\frac{2\log\frac{2}{\delta}}{m}} \log \sqrt{\frac{m^3}{2\log\frac{2}{\delta}}} \nonumber \\
        &\le s \sqrt{\frac{2\log\frac{2}{\delta}}{m}} \prn*{\frac{m^3}{2\log\frac{2}{\delta}}}^{\frac{1}{2s}} \le sm^{\frac{3}{2s}-\frac{1}{2}}\sqrt{2\log\frac{2}{\delta}} \label{eq:concentrate2},
    \end{align}
    where (\ref{eq:concentrate1}) is maximized when $\abs{\lambda_1 - \mu_1} = \cdots = \abs{\lambda_n - \mu_n} = \frac{1}{m}\sqrt{\frac{2\log\frac{2}{\delta}}{m}}$, and (\ref{eq:concentrate2}) follows by the fact that for any $t > 0$, $\log x \le x^t/t$. By taking $s = 9$, we have
    \begin{align*}
        \abs{S_1(X) - \hat{S}_1(X)} &= C_\kappa \abs*{\tr(G_X \log G_X) - \tr(K \log K)} \\
        &= C_\kappa \abs*{\tr(G_X \log G_X) - \tr(\hat{G}_X \log \hat{G}_X)} \le \frac{9 C_\kappa \sqrt{2\log\frac{2}{\delta}}}{\sqrt[3]{m}}.
    \end{align*}
\end{proof}

\subsection{Proof of Proposition \ref{prop:finite}}
\begin{restateproposition}{\ref{prop:finite}}[Restate]
    Let $\kappa(\x, \x^\prime) = \mathbbm{1}_{\norm{\x - \x^\prime} < c}$. Assume that the PDF $p_X(\cdot)$ satisfies:
    \begin{itemize}
        \item Continuous: $\forall \x \in \mathcal{X}$, $\lim_{\x^\prime \rightarrow \x} p_X(\x^\prime) = p_X(\x)$;
        \item Positive: $\forall \x \in \mathcal{X}$, $\lim_{\x^\prime \rightarrow \x} p_X(\x^\prime) > 0$;
    \end{itemize}
    then we have $\lim_{c \rightarrow 0} E_X^\kappa \rightarrow 0$ and $\lim_{c \rightarrow 0} {E_X^\kappa}^\prime \rightarrow 0$.
\end{restateproposition}
\begin{proof}
    For any PDF $p(\cdot)$ defined on $\mathcal{X}$:
    \begin{align*}
        \lim_{c \rightarrow 0} E_X^\kappa(p) &= \lim_{c \rightarrow 0} C_\kappa \iint_{\mathcal{X}^2} p(\x) \brk*{\log p_X(\x) - \log p_X(\x^\prime)} \kappa^2(\x, \x^\prime) \dif \x \dif \x^\prime \\
        &= \lim_{c \rightarrow 0} C_\kappa \iint_{\mathcal{X}^2} p(\x) \brk*{\log p_X(\x) - \log p_X(\x + \x^\prime)} \kappa^2(0, \x^\prime) \dif \x \dif \x^\prime \\
        &= C_\kappa \int_\mathcal{X} \brc*{\lim_{c \rightarrow 0} \int_{\norm{\x - \x^\prime}<c} p(\x) \brk*{\log p_X(\x) - \log p_X(\x + \x^\prime)} \dif \x} \dif \x^\prime \\
        &= C_\kappa \int_\mathcal{X} 0 \dif \x^\prime = 0.
    \end{align*}
\end{proof}

\subsection{Proof of Proposition \ref{prop:property}}
\begin{restateproposition}{\ref{prop:property}}[Restate]
    Let $X, X^\prime \in \mathcal{X}$, $Y \in \mathcal{Y}$, $Z \in \mathcal{Z}$ be continuous random variables with probability measures $P_X$, $P_{X^\prime}$, $P_Y$ and $P_Z$ respectively. Then
    \begin{enumerate}
        \item $H(X) \le S_1(X) \le H(X) + {E_X^\kappa}^\prime$.
        \item $D_1(P_X \parallel P_{X^\prime}) \ge -E_X^\kappa$.
        \item $I_1(X; Y) = D_1(P_{X,Y} \parallel P_X \otimes P_Y) \ge 0$.
        \item $I_1(X; Y) = S_1(X) + S_1(Y) - S_1(X, Y)$.
        \item $I_1(X;Y|Z) = I_1(X;Y,Z) - I_1(X;Z)$.
        \item Let $X, Y, Z$ form Markov chain $X \rightarrow Y \rightarrow Z$, then $I_1(X;Y) \ge I_1(X;Z)$ and $I_1(Y;Z) \ge I_1(X;Z)$.
    \end{enumerate}
\end{restateproposition}
\begin{proof}[Proof of Property \ref{prop:ke_bound}]
    \begin{align}
        S_1(X) - H(X) &= \int_\mathcal{X} p_X(\x) \log p_X(\x) \dif \x - C_\kappa \iint_{\mathcal{X}^2} p_X(\x) \log p_X(\x^\prime) \kappa^2(\x, \x^\prime) \dif \x \dif \x^\prime \nonumber\\
        &= \int_\mathcal{X} p_X(\x) \log p_X(\x) \dif \x - \int_\mathcal{X} p_X(\x) \prn*{C_\kappa \int_\mathcal{X} \log p_X(\x^\prime) \kappa^2(\x, \x^\prime) \dif \x^\prime} \dif \x \nonumber\\
        &\ge \int_\mathcal{X} p_X(\x) \log p_X(\x) \dif \x - \int_\mathcal{X} p_X(\x) \prn*{\log C_\kappa \int_\mathcal{X} p_X(\x^\prime) \kappa^2(\x, \x^\prime) \dif \x^\prime} \dif \x \label{eq:ke_lipschitz1}\\
        &= \int_\mathcal{X} p_X(\x) \log \frac{p_X(\x)}{C_\kappa \int_\mathcal{X} p_X(\x^\prime) \kappa^2(\x, \x^\prime) \dif \x^\prime} \dif \x \nonumber\\
        &= \KL(P_X \parallel Q_X) \ge 0, \nonumber
    \end{align}
    where (\ref{eq:ke_lipschitz1}) follows by Jensen's inequality, and $Q_X$ is the probability measure with PDF $q_X(\x) = C_\kappa \int_\mathcal{X} p_X(\x^\prime) \kappa^2(\x, \x^\prime) \dif \x^\prime$. Meanwhile,
    \begin{align*}
        S_1(X) &= -C_\kappa \iint_{\mathcal{X}^2} p_X(\x) \log p_X(\x^\prime) \kappa^2(\x, \x^\prime) \dif \x \dif \x^\prime \nonumber\\
        &= -\int_\mathcal{X} p_X(\x) \prn*{C_\kappa \int_\mathcal{X} \log p_X(\x^\prime) \kappa^2(\x, \x^\prime) \dif \x^\prime} \dif \x \nonumber\\
        &= -\int_\mathcal{X} p_X(\x) \log p_X(\x) \dif \x - \int_\mathcal{X} p_X(\x) \brk*{C_\kappa \int_\mathcal{X} \prn*{\log p_X(\x^\prime) - \log p_X(\x)} \kappa^2(\x, \x^\prime) \dif \x^\prime} \dif \x \nonumber\\
        &\le H(X) + {E_X^\kappa}^\prime. \nonumber
    \end{align*}
\end{proof}
\begin{proof}[Proof of Property \ref{prop:kl_positive}]
    Let $p$ and $q$ be the PDF of $X$ and $X^\prime$ respectively. Then consider the following functional on $q(\x)$:
    \begin{equation*}
        J(q) = C_\kappa \iint_{\mathcal{X}^2} p(\x) \log\frac{p(\x^\prime)}{q(\x^\prime)} \kappa^2(\x, \x^\prime) \dif \x \dif \x^\prime - \eta_0 \prn*{\int_\mathcal{X} q(\x) \dif \x - 1},
    \end{equation*}
    where $\eta_0$ is the Lagrange multiplier that ensure $q(\x)$ is a probability distribution. The divergence $D_1(P \parallel Q)$ attains an extremum when the functional derivative is equal to zero:
    \begin{equation*}
        \frac{\partial J}{\partial q} = -C_\kappa \int_\mathcal{X} \frac{p(\x)}{q(\x^\prime)} \kappa^2(\x, \x^\prime) \dif \x - \eta_0 = -\frac{C_\kappa \int_\mathcal{X} p(\x) \kappa^2(\x, \x^\prime) \dif \x}{q(\x^\prime)} - \eta_0 = 0,
    \end{equation*}
    which indicates that the minimizer $\hat{q}(\x^\prime)$ satisfies $\hat{q}(\x^\prime) \propto C_\kappa \int_\mathcal{X} p(\x) \kappa^2(\x, \x^\prime) \dif \x$. Combing with $\int_\mathcal{X} \hat{q}(\x^\prime) \dif \x^\prime = 1$, we have
    \begin{equation*}
        \hat{q}(\x^\prime) = C_\kappa \int_\mathcal{X} p(\x) \kappa^2(\x, \x^\prime) \dif \x.
    \end{equation*}
    Therefore,
    \begin{align}
        D_1(P \parallel Q) &\ge C_\kappa \iint_{\mathcal{X}^2} p(\x) \log\frac{p(\x^\prime)}{\hat{q}(\x^\prime)} \kappa^2(\x, \x^\prime) \dif \x \dif \x^\prime \nonumber \\
        &= \int_\mathcal{X} \prn*{C_\kappa \int_\mathcal{X} p(\x) \kappa^2(\x, \x^\prime) \dif \x} \log\frac{p(\x^\prime)}{\hat{q}(\x^\prime)} \dif \x^\prime \nonumber \\
        &\ge -\int_\mathcal{X} \hat{q}(\x^\prime) \prn*{C_\kappa \int_\mathcal{X} (\log p(\x^\prime) - \log p(\x)) \kappa^2(\x, \x^\prime) \dif \x} \dif \x^\prime \label{eq:kl_positive1} \\
        &\ge -E_X^\kappa. \nonumber
    \end{align}
    where (\ref{eq:kl_positive1}) follows by Jensen's inequality.
\end{proof}
\begin{proof}[Proof of Property \ref{prop:mi_kl}]
    The first equality directly follows from the definition of kernelized divergence (Definition \ref{def:kd}) and mutual information (Definition \ref{def:kmi}). The positiveness of $I_1(X; Y)$ follows by setting $n \rightarrow \infty$ in Proposition 4.1 of \cite{giraldo2014measures}.
\end{proof}
\begin{proof}[Proof of Property \ref{prop:mi_ke}]
    Notice that
    \begin{align*}
        S_1(X) &= -C_{\kappa_X} \iint_{\mathcal{X}^2} p_X(\x) \log p_X(\x^\prime) \kappa_X^2(\x, \x^\prime) \dif \x \dif \x^\prime \\
        &= -C_{\kappa_X} \iiint_{\mathcal{Y} \times \mathcal{X}^2} p_{X,Y}(\x, \y) \log p_X(\x^\prime) \kappa_X^2(\x, \x^\prime) \dif \x \dif \x^\prime \dif \y \\
        &= -C_{\kappa_X} C_{\kappa_Y} \iiiint_{\mathcal{Y}^2 \times \mathcal{X}^2} p_{X,Y}(\x, \y) \log p_X(\x^\prime) \kappa_X^2(\x, \x^\prime) \kappa_Y^2(\y, \y^\prime) \dif \x \dif \x^\prime \dif \y \dif \y^\prime.
    \end{align*}
    Similarly, we have $S_1(Y) = -C_{\kappa_X} C_{\kappa_Y} \iiiint_{\mathcal{Y}^2 \times \mathcal{X}^2} p_{X,Y}(\x, \y) \log p_Y(\y^\prime) \kappa_X^2(\x, \x^\prime) \kappa_Y^2(\y, \y^\prime) \dif \x \dif \x^\prime \dif \y \dif \y^\prime$. Combining the expressions above finishes the proof.
\end{proof}
\begin{proof}[Proof of Property \ref{prop:cond_mi}]
    Following the definition of kernelized mutual information, we can derive the expression of kernelized conditional mutual information as follows:
    \begin{align}
        I_1(X; Y|Z) &= C_{\kappa_X} C_{\kappa_Y} C_{\kappa_Z} \iint_{\mathcal{Z}^2} \iint_{\mathcal{Y}^2} \iint_{\mathcal{X}^2} p_{X,Y,Z}(x, y, z) \log \frac{p_{X,Y|Z}(x^\prime, y^\prime| z^\prime)}{p_{X|Z}(x^\prime| z^\prime)p_{Y|Z}(y^\prime| z^\prime)} \\
        &\qquad\cdot \kappa_X^2(x, x^\prime) \kappa_Y^2(y, y^\prime) \kappa_Z^2(z, z^\prime) \dif x \dif x^\prime \dif y \dif y^\prime \dif z \dif z^\prime. \label{def:cmi}
    \end{align}
    Then the proof of property \ref{prop:cond_mi} directly follows by the definition of conditional and unconditional kernelized mutual information.
\end{proof}
\begin{proof}[Proof of Property \ref{prop:data_proc}]
    Following property \ref{prop:cond_mi}, we have that
    \begin{equation*}
        I_1(X;Y,Z) = I_1(X;Y|Z) + I_1(X;Z) = I_1(X;Z|Y) + I_1(X;Y).
    \end{equation*}
    From the Markov condition, $X$ and $Z$ are conditionally independent given $Y$, i.e. $p_{X,Z|Y}(x,z|y) = p_{X|Y}(x|y) p_{Z|Y}(z|y)$. By the definition of conditional mutual information, we have $I_1(X;Z|Y) = 0$. Then the first inequality of property \ref{prop:data_proc} follows by the positiveness of $I_1(X;Y|Z)$. Similarly, one can prove the second part of property \ref{prop:data_proc}.
\end{proof}

\section{Proof of Section \ref{sec:gen}}

\subsection{Proof of Theorem \ref{thm:gen_mi}}
\begin{lemma} \label{lm:dv}
    Let $P$, $Q$ be probability measures defined on the same measurable space, where $P$ is absolutely continuous with respect to $Q$. Then
    \begin{equation*}
        D_1(P \parallel Q) + E_P^\kappa \ge \E_P[X] - \log \E_Q[e^X].
    \end{equation*}
    where $X$ is any random variable such that $e^X$ is $Q$-integrable and $\E_P[X]$ exists.
\end{lemma}
\begin{proof}
    Define $Q^X$ be a probability measure such that
    \begin{equation*}
        Q^X(\Omega) = \int_\Omega \frac{e^X}{\E_Q[e^X]} \dif Q,
    \end{equation*}
    then $Q$ is absolutely continuous with respect to $Q^X$. Observe that
    \begin{align*}
        D_1(P \parallel Q) + E_P^\kappa &= D_1(P \parallel Q^X) + E_P^\kappa + C_\kappa \iint_{\mathcal{X}^2} p_X(x) \log \frac{e^{x^\prime}}{\E_Q[e^X]} \kappa^2(x, x^\prime) \dif x \dif x^\prime \\
        &\ge C_\kappa \iint_{\mathcal{X}^2} p_X(x) \log e^{x^\prime} \kappa^2(x, x^\prime) \dif x \dif x^\prime - C_\kappa \iint_{\mathcal{X}^2} p_X(x) \log \E_Q[e^X] \kappa^2(x, x^\prime) \dif x \dif x^\prime \\
        &= \int_\mathcal{X} p_X(x) \prn*{C_\kappa \int_\mathcal{X} x^\prime \kappa^2(x, x^\prime) \dif x^\prime} \dif x - \log \E_Q[e^X] \int_\mathcal{X} p_X(x) \prn*{C_\kappa \int_\mathcal{X} \kappa^2(x, x^\prime) \dif x^\prime} \dif x \\
        & = \int_\mathcal{X} p_X(x) x \dif x - \log \E_Q[e^X] \int_\mathcal{X} p_X(x) \dif x = \E_P[X] - \log \E_Q[e^X].
    \end{align*}
\end{proof}

\begin{lemma} (Lemma 2 in \cite{harutyunyan2021information}) \label{lm:exp_subgauss}
    Let $X$ be a zero-mean random variable that is $R$-subgaussian, then $\forall \lambda \in \left[0, \frac{1}{4R^2}\right)$:
    \begin{equation*}
        \E\brk*{e^{\lambda X^2}} \le 1 + 8\lambda R^2.
    \end{equation*}
\end{lemma}

\begin{lemma} (Lemma 3 in \cite{harutyunyan2021information}) \label{lm:subgauss}
    Let $X$ and $Y$ be independent random variables. Let $f$ be a measurable function such that $f(x, Y)$ is $R$-subgaussian and $\E_Y[f(x, Y)] = 0$ for all $x \in \mathcal{X}$, then $f(X, Y)$ is also $R$-subgaussian.
\end{lemma}

\begin{restatetheorem}{\ref{thm:gen_mi}}[Restate]
    Suppose that $\ell(w, Z)$ is $R$-subgaussian with respect to $Z$ for every $w \in \mathcal{W}$, then
    \begin{align*}
        \abs{\E_{S,W}[L(W) - L_S(W)]} &\le \sqrt{\frac{2R^2 \hat{I}_1(S; W)}{n}}, \quad \textrm{and} \\
        \E_{S,W}[L(W) - L_S(W)]^2 &\le \frac{4R^2(\hat{I}_1(S; W) + \log 3)}{n},
    \end{align*}
    where $\hat{I}_1(S; W) = I_1(S; W) + E_{S,W}^\kappa$.
\end{restatetheorem}
\begin{proof}
    Let $f(w, s) = L(w) - L_s(w)$ and let $W^\prime$ and $S^\prime$ be independent copies of $W$ and $S$ and $\lambda \in [0, \infty)$, then
    \begin{align}
        I_1(W; S) + E_{W,S}^\kappa &= D_1(P_{W,S} \parallel P_W \otimes P_S) + E_{W,S}^\kappa \nonumber\\
        &\ge \E_{W,S}[\lambda f(W,S)] - \log \E_{W^\prime, S^\prime}\brk*{e^{\lambda f(W^\prime, S^\prime)}} \nonumber\\
        &= \E_{W,S}[\lambda f(W,S)] - \log \E_{W, S^\prime}\brk*{e^{\lambda f(W, S^\prime)}} \label{eq:gen_mi1}
    \end{align}
    by Lemma \ref{lm:dv} and the fact that $W, W^\prime, S^\prime$ are independent. Notice that $f(w, S)$ is $R/\sqrt{n}$-subgaussian for each $w \in \mathcal{W}$, since $L_S(w)$ is the average of $n$ i.i.d. $R$-subgaussian random variables. Moreover, $\E_S[f(w, S)] = 0$ for any fixed $w$. Then by Lemma \ref{lm:subgauss}, $f(W, S)$ is $R/\sqrt{n}$-subgaussian. Therefore,
    \begin{gather*}
        \log \E_{W, S^\prime}\brk*{e^{\lambda f(W,S^\prime) - \lambda \E_{W, S^\prime}[f(W,S^\prime)]}} \le \frac{\lambda^2R^2}{2n}, \\
        \log \E_{W, S^\prime}\brk*{e^{\lambda f(W,S^\prime)}} \le \frac{\lambda^2R^2}{2n}.
    \end{gather*}
    Plugging into (\ref{eq:gen_mi1}), we have
    \begin{equation*}
        I_1(W; S) + E_{W,S}^\kappa \ge \lambda\E_{W,S}[f(W,S)] - \frac{\lambda^2R^2}{2} \ge \frac{n}{2R^2} \E_{W,S}^2[f(W,S)].
    \end{equation*}
    This finishes the proof of the first part. For the second part, let $\tilde{f}(w, s) = \prn*{L(w) - L_s(w)}^2$ and $\lambda \in \left[0, \frac{1}{4R^2}\right)$, then
    \begin{align*}
        I_1(W; S) + E_{W,S}^\kappa &\ge \E_{W,S}\brk*{\lambda \tilde{f}(W,S)} - \log \E_{W^\prime, S^\prime}\brk*{e^{\lambda \tilde{f}(W^\prime, S^\prime)}} \\
        &\ge \E_{W,S}\brk*{\lambda \prn*{L(W) - L_S(W)}^2} - \log (1 + 8\lambda R^2) \\
        &\ge \frac{1}{4R^2} \E_{W,S}\brk*{\prn*{L(W) - L_S(W)}^2} - \log 3,
    \end{align*}
    by Lemma \ref{lm:dv}, \ref{lm:exp_subgauss} and taking $\lambda = \frac{1}{4R^2}$. This finishes the proof of the second part.
\end{proof}

\subsection{Proof of Theorem \ref{thm:gen_sgld}}
\begin{lemma} \label{lm:max_ke}
    Let random variable $X \sim N(0, \Sigma)$ and $X^\prime$ be any continuous random variable that satisfies $\E[X^\prime] = 0$ and $\Cov[X^\prime] = \Sigma$, then $S_1(X^\prime) \le S_1(X) + E_{X^\prime}^\kappa$. Furthermore if $\kappa$ is Gaussian with kernel width $\sigma_\kappa$, then $S_1(X) = \frac{d}{2} \log(2\pi e) + \frac{1}{2}\log\abs{\Sigma} + \frac{\sigma_\kappa^2}{4}\tr[\Sigma^{-1}]$.
\end{lemma}
\begin{proof}
    Let $p(\cdot)$ and $q(\cdot)$ be the PDF of $X$ and $X^\prime$ respectively. Notice that $p_\kappa = C_\kappa \kappa^2(0, \cdot)$ integrates to $1$, thus could be treated as a probability distribution whose covariance matrix is $\frac{1}{2}\sigma_\kappa^2I$. Then we have
    \begin{align*}
        S_1(X) - S_1(X^\prime) &= C_\kappa \iint_{\mathcal{X}^2} \brk*{q(\x) \prn*{\log p(\x^\prime) + \log \frac{q(\x^\prime)}{p(\x^\prime)}} - p(\x) \log p(\x^\prime)} \kappa^2(\x, \x^\prime) \dif \x \dif \x^\prime \\
        &= C_\kappa \iint_{\mathcal{X}^2} \brk*{q(\x) - p(\x)} \log p(\x^\prime) \kappa^2(\x, \x^\prime) \dif \x \dif \x^\prime + C_\kappa \iint_{\mathcal{X}^2} q(\x) \log \frac{q(\x^\prime)}{p(\x^\prime)} \kappa^2(\x, \x^\prime) \dif \x \dif \x^\prime \\
        &= C_\kappa \iint_{\mathcal{X}^2} \brk*{q(\x) - p(\x)} \prn*{-\frac{d}{2} \log(2\pi) -\frac{1}{2}\log\abs{\Sigma} - \frac{1}{2}{\x^\prime}^\top\Sigma^{-1}\x^\prime} \kappa^2(\x, \x^\prime) \dif \x \dif \x^\prime + D_1(Q \parallel P) \\
        &\ge -\prn*{\frac{d}{2} \log(2\pi) +\frac{1}{2}\log\abs{\Sigma}} \int_\mathcal{X} \brk*{q(\x) - p(\x)} \prn*{C_\kappa \int_\mathcal{X} \kappa^2(\x, \x^\prime) \dif \x^\prime} \dif \x \\
        &\qquad- \frac{1}{2} \tr\brc*{\brk*{\int_\mathcal{X} \brk*{q(\x) - p(\x)} \prn*{C_\kappa \int_\mathcal{X}\x^\prime{\x^\prime}^\top \kappa^2(\x, \x^\prime) \dif \x^\prime} \dif \x}\Sigma^{-1}} - E_{X^\prime}^\kappa \\
        &= -\prn*{\frac{d}{2} \log(2\pi) +\frac{1}{2}\log\abs{\Sigma}} \int_\mathcal{X} \brk*{q(\x) - p(\x)} \dif \x \\
        &\qquad- \frac{1}{2} \tr\brc*{\brk*{\int_\mathcal{X} \brk*{q(\x) - p(\x)} \prn*{\x\x^\top + \frac{1}{2}\sigma_\kappa^2I} \dif \x}\Sigma^{-1}} - E_{X^\prime}^\kappa \\
        &= -\frac{1}{2} \tr\brk*{\prn*{\int_\mathcal{X} \brk*{q(\x) - p(\x)} \x\x^\top \dif \x}\Sigma^{-1}} - E_{X^\prime}^\kappa = -\frac{1}{2} \tr\brk*{\prn*{\Sigma - \Sigma}\Sigma^{-1}} - E_{X^\prime}^\kappa =  - E_{X^\prime}^\kappa.
    \end{align*}
    Consider the case that kernel $\kappa$ is Gaussian, i.e. $\kappa^2(\x, \x^\prime) = \exp\prn*{-\norm{\x-\x^\prime}_2^2/\sigma_\kappa^2}$ and $C_\kappa = (\pi\sigma_\kappa^2)^{-d/2}$, we have
    \begin{align*}
        S_1(X) &= \frac{1}{\sqrt{(2\pi)^d\abs{\Sigma}}} \frac{1}{\sqrt{(\pi\sigma_\kappa^2)^d}} \iint_{\mathcal{X}^2} \exp\prn*{-\frac{1}{2} \x^\top \Sigma^{-1} \x} \brk*{\frac{d}{2} \log(2\pi) +\frac{1}{2}\log\abs{\Sigma} + \frac{1}{2} {\x^\prime}^\top\Sigma^{-1}\x^\prime} \exp\prn*{-\frac{\norm{\x-\x^\prime}_2^2}{\sigma_\kappa^2}} \dif \x \dif \x^\prime \\
        &= \frac{d}{2} \log(2\pi) + \frac{1}{2}\log\abs{\Sigma} + \frac{1}{\sqrt{(2\pi)^d\abs{\Sigma}}} \int_\mathcal{X} \frac{1}{2} \exp\prn*{-\frac{1}{2} \x^\top \Sigma^{-1} \x} \prn*{\frac{1}{\sqrt{(\pi\sigma_\kappa^2)^d}} \int_\mathcal{X} {\x^\prime}^\top\Sigma^{-1}\x^\prime \exp\prn*{-\frac{\norm{\x-\x^\prime}_2^2}{\sigma_\kappa^2}} \dif \x^\prime} \dif \x \\
        &= \frac{d}{2} \log(2\pi) + \frac{1}{2}\log\abs{\Sigma} + \frac{1}{\sqrt{(2\pi)^d\abs{\Sigma}}} \int_\mathcal{X} \frac{1}{2} \exp\prn*{-\frac{1}{2} \x^\top \Sigma^{-1} \x} \tr\brk*{\prn*{\x\x^\top + \frac{1}{2}\sigma_\kappa^2I}\Sigma^{-1}} \dif \x \\
        &= \frac{d}{2} \log(2\pi) + \frac{1}{2}\log\abs{\Sigma} + \frac{1}{2}\tr\brk*{\prn*{\Sigma + \frac{1}{2}\sigma_\kappa^2I}\Sigma^{-1}} = \frac{d}{2} \log(2\pi e) + \frac{1}{2}\log\abs{\Sigma} + \frac{\sigma_\kappa^2}{4}\tr[\Sigma^{-1}].
    \end{align*}
\end{proof}

\begin{lemma} \label{lm:mi_bound1}
    Let $X$, $Y$, $\Delta$ and $\xi \sim N(0, \sigma^2I)$ be independent random variables. Let $f$: $\mathcal{W} \times \mathcal{Z}^b \rightarrow \mathcal{W}$ be a determinant function and let $\Omega(\cdot) = \E_X[f(\cdot,X)]$. Then
    \begin{align*}
        I_1(f(Y+\Delta,X) + \xi;X|Y) &\le \frac{1}{2}\log\abs*{\frac{1}{\sigma^2} \Cov[g(Y,\Delta,X)] + I} + E_{f(Y+\Delta,X)-\Omega(Y+\Delta)+\xi|Y,\Delta}^\kappa.
    \end{align*}
\end{lemma}
\begin{proof}
    Let $g(Y,\Delta,X) = f(Y+\Delta,X) - \Omega(Y+\Delta)$, then
    \begin{align}
        &\quad I_1(f(Y+\Delta,X) + \xi;X|Y = \y,\Delta = \delta) \\
        &= I_1(g(Y,\Delta,X) + \xi;X|Y = \y,\Delta = \delta) \nonumber \\
        &= S_1(g(Y,\Delta,X) + \xi|Y = \y,\Delta = \delta) - S_1(g(Y,\Delta,X) + \xi|X,Y = \y,\Delta = \delta) \nonumber \\
        &= S_1(g(Y,\Delta,X) + \xi|Y = \y,\Delta = \delta) - S_1(\xi) \nonumber \\
        &= S_1(g(Y,\Delta,X) + \xi|Y = \y,\Delta = \delta) - \frac{d}{2} \log(2\pi e \sigma^2) - \frac{d\sigma_\kappa^2}{4\sigma^2} \nonumber \\
        &\le \frac{d}{2} \log(2\pi e) + \frac{1}{2}\log\abs*{\Cov[g(Y,\Delta,X)|Y=\y,\Delta=\delta] + \sigma^2I} - \frac{d}{2} \log(2\pi e \sigma^2) + E_{g(Y,\Delta,X)+\xi|Y,\Delta}^\kappa \nonumber \\
        &\qquad+ \frac{\sigma_\kappa^2}{4}\tr\brk*{\prn*{\Cov[g(Y,\Delta,X)|Y=\y,\Delta=\delta] + \sigma^2I}^{-1}} - \frac{d\sigma_\kappa^2}{4\sigma^2} \label{eq:mi_bound2} \\
        &\le \frac{1}{2}\log\abs*{\frac{1}{\sigma^2} \Cov[g(Y,\Delta,X)|Y=\y,\Delta=\delta] + I} + E_{g(Y,\Delta,X)+\xi|Y,\Delta}^\kappa \label{eq:mi_bound3}
    \end{align}
    where (\ref{eq:mi_bound2}) follows by Lemma \ref{lm:max_ke} and noticing that
    \begin{align*}
        \Cov[g(Y,\Delta,X)+\xi|Y=\y,\Delta=\delta] &= \Cov[g(Y,\Delta,X)|Y=\y,\Delta=\delta] + \Cov[\xi] \\
        &= \Cov[g(Y,\Delta,X)|Y=\y,\Delta=\delta] + \sigma^2I,
    \end{align*}
    and (\ref{eq:mi_bound3}) follows by the fact that
    \begin{equation*}
        \tr\brk*{\prn*{\Cov[g(Y,\Delta,X)|Y=\y,\Delta=\delta] + \sigma^2I}^{-1}} \le \tr\brk*{\prn*{\sigma^2I}^{-1}}.
    \end{equation*}
    This leads to the following bound:
    \begin{align}
        I_1(f(Y+\Delta,X) + \xi;X|Y) &\le I_1(f(Y+\Delta,X) + \xi,\Delta;X|Y) \\
        &= I_1(f(Y+\Delta,X) + \xi,\Delta;X|Y) - I_1(\Delta;X|Y) \label{eq:mi_bound2_1} \\
        &= I_1(f(Y+\Delta,X) + \xi;X|Y,\Delta) \\
        &= \E_{Y,\Delta}\brk*{I_1(f(Y+\Delta,X) + \xi;X|Y=\y,\Delta=\delta)} \nonumber \\
        &\le \E_{Y,\Delta}\brk*{\frac{1}{2}\log\abs*{\frac{1}{\sigma^2} \Cov[g(Y,\Delta,X)|Y=\y,\Delta=\delta] + I}} + E_{g(Y,\Delta,X)+\xi|Y,\Delta}^\kappa \label{eq:mi_bound2_2} \\
        &\le \frac{1}{2}\log\abs*{\frac{1}{\sigma^2} \E_{Y,\Delta}\brk*{\Cov[g(Y,\Delta,X)|Y=\y,\Delta=\delta]} + I} + E_{g(Y,\Delta,X)+\xi|Y,\Delta}^\kappa \label{eq:mi_bound2_3} \\
        &= \frac{1}{2}\log\abs*{\frac{1}{\sigma^2} \Cov[g(Y,\Delta,X)] + I} + E_{g(Y,\Delta,X)+\xi|Y,\Delta}^\kappa, \label{eq:mi_bound2_4}
    \end{align}
    where (\ref{eq:mi_bound2_1}) follows by the fact that $\Delta$ and $X$ are independent, (\ref{eq:mi_bound2_2}) follows by applying (\ref{eq:mi_bound3}), (\ref{eq:mi_bound2_3}) follows by Jensen's inequality and the concavity of the log-determinant function, and (\ref{eq:mi_bound2_4}) follows by the law of total variance and noticing that
    \begin{equation*}
        \Cov\brk*{\E_X[g(Y,\Delta,X)|Y=\y,\Delta=\delta]} = \Cov\brk*{\E_X[f(Y,\Delta,X) - \E_X[f(Y,\Delta,X)]|Y=\y,\Delta=\delta]} = 0.
    \end{equation*}
\end{proof}

\begin{restatetheorem}{\ref{thm:gen_sgld}}[Restate]
    Under the same conditions of Lemma \ref{lm:gen_sgld}:
    \begin{gather*}
        I_1(W_T;S) \le \sum_{t=1}^T I_1\prn*{W_t;B_t|W_{t-1}} \le \sum_{t=1}^T \prn*{\frac{1}{2} \log\abs*{\frac{\eta_t^2}{\sigma_t^2}\V_t + I} + E_{W_t|W_{t-1}}^\kappa}, \\
        I(W_T;S) \le \sum_{t=1}^T \frac{1}{2} \log\abs*{\frac{\eta_t^2}{\sigma_t^2}\V_t + I},
    \end{gather*}
    where $\V_t = \Cov[g(W_{t-1}, B_t)]$ is the \textbf{gradient covariance} matrix and $\abs{\cdot}$ denotes the matrix determinant.
\end{restatetheorem}
\begin{proof}
    Notice that $S \rightarrow (B_1, \cdots, B_T) \rightarrow (W_1, \cdots, W_T)$ form a Markov chain, then
    \begin{align*}
        I_1(W_T;S) &\le I_1(W_T; B_1, \cdots, B_T) \le I_1(W_0, W_1, \cdots, W_T; B_1, \cdots, B_T) \\
        &= I_1(W_0; B_1, \cdots, B_T) + I_1(W_1; B_1, \cdots, B_T| W_0) + I_1(W_2; B_1, \cdots, B_T| W_0, W_1) \\
        &\qquad + \cdots + I_1(W_T; B_1, \cdots, B_T| W_0, \cdots, W_T).
    \end{align*}
    For each $t \in [1, T]$, we have
    \begin{align*}
        I_1(W_t; B_1, \cdots, B_t| W_0, \cdots, W_{t-1}) &= S_1(W_t| W_0, \cdots, W_{t-1}) - S_1(W_t| B_1, \cdots, B_t, W_0, \cdots, W_{t-1}) \\
        &= S_1(W_t| W_{t-1}) - S_1(W_t| B_t, W_{t-1}) \\
        &= I_1(W_t; B_t| W_{t-1}).
    \end{align*}
    Let $X = B_t$, $Y = W_{t-1}$, $\Delta = 0$, $\xi = \xi_t$ and $f(W_{t-1},B_t) = -\eta_t g(W_{t-1},B_t)$ in Lemma \ref{lm:mi_bound1}, we then have
    \begin{align*}
        I_1(W_t; B_t| W_{t-1}) &= I_1(W_t - W_{t-1}; B_t| W_{t-1}) = I_1(-\eta_t g(W_{t-1},B_t) + \xi_t; B_t| W_{t-1}) \\
        &\le \frac{1}{2}\log\abs*{\frac{\eta_t^2}{\sigma_t^2} \Cov[g(W_{t-1},B_t)] + I} + E_{W_t|W_{t-1}}^\kappa,
    \end{align*}
    which finishes the proof.
\end{proof}

\subsection{Proof of Proposition \ref{prop:tight}}
\begin{lemma} \label{lm:det}
    Let $V$ be an $n \times n$ symmetric positive definite matrix partitioned by
    \begin{equation*}
        V = \begin{bmatrix}
            A & C^\top \\
            C & B
        \end{bmatrix},
    \end{equation*}
    where $A$, $B$ are symmetric matrices of size $n_1 \times n_1$ and $n_2 \times n_2$ respectively. Then $\abs{V} \le \abs{A}\abs{B}$.
\end{lemma}
\begin{proof}
    Notice that
    \begin{equation*}
        V = D\begin{bmatrix}
            A & 0 \\
            0 & B - CA^{-1}C^\top
        \end{bmatrix}D^\top, \quad \textrm{where } D = \begin{bmatrix}
            I_{n_1} & 0 \\
            CA^{-1} & I_{n_2}
        \end{bmatrix},
    \end{equation*}
    then we have $\abs{V} = \abs{D}\abs{A}\abs{B - CA^{-1}C^\top}\abs{D^\top} = \abs{A}\abs{B - CA^{-1}C^\top}$. Let $D^\dagger$ be the pseudo-inverse of $D$, then for any column vector $\x$ of length $n$:
    \begin{equation*}
        \x^\top \begin{bmatrix}
            A & 0 \\
            0 & B - CA^{-1}C^\top
        \end{bmatrix} \x = (\x^\top D^\dagger) V (\x^\top D^\dagger)^\top \ge 0,
    \end{equation*}
    therefore $B - CA^{-1}C^\top$ is positive semi-definite. Similarly, we can prove that $CA^{-1}C^\top$ is positive semi-definite. Let $\lambda_i$, $\mu_i$, $\nu_i$, $i \in \{1, \cdots, n_2\}$ be the eigenvalues of $B - CA^{-1}C^\top$, $B$ and $CA^{-1}C^\top$ respectively in descending order, then by Weyl's inequality, we have $\lambda_i \le \mu_i - \nu_{n_2} \le \mu_i$ for all $i \in \{1, \cdots, n_2\}$, which implies that $\abs{B - CA^{-1}C^\top} \le \abs{B}$. The proof is complete by combining the result above: $\abs{V} = \abs{A}\abs{B - CA^{-1}C^\top} \le \abs{A}\abs{B}$.
\end{proof}

\begin{restateproposition}{\ref{prop:tight}}[Restate]
    Given $\V_t$, $V_t$ and $L$ defined as above, let $\{c_i\}_{i=1}^r$ be a partition of $\{n\}$, i.e. $c_1 \cup \cdots \cup c_r = \{n\}$ and $c_i \cap c_j = \Phi$ for any $1 \le i < j \le r$. Let $\V_t^i$ be the sub-matrix of $\V_t$ with columns and rows indexed by $c_i$, and define
    \begin{gather*}
        \theta_c(\V) = \frac{1}{2}\log\abs*{\frac{\eta_t^2}{\sigma_t^2}\V+I}, \quad \theta_v(V) = \frac{d}{2}\log\prn*{\frac{\eta_t^2V}{d\sigma_t^2}+1}, \\
        \textrm{then}\qquad\theta_c(\V_t) \le \sum_{i=1}^r\theta_c(\V_t^i) \le \theta_v(V_t) \le \theta_v(L).
    \end{gather*}
\end{restateproposition}
\begin{proof}
    Notice that $V_t = \tr[\V_t]$. Since the covariance matrix is always symmetric positive semi-definite, we can denote the eigenvalues of $\V_t$ by $\lambda_1$, $\cdots$, $\lambda_d \ge 0$, then
    \begin{equation*}
        \log\abs*{\frac{\eta_t^2}{\sigma_t^2}\V_t+I} = \log\brk*{\prod_{i=1}^d \prn*{\frac{\eta_t^2\lambda_i}{\sigma_t^2} + 1}} \le \log\brk*{\frac{1}{d}\sum_{i=1}^d \prn*{\frac{\eta_t^2\lambda_i}{\sigma_t^2} + 1}}^d = d\log\brk*{\frac{\eta_t^2}{d\sigma_t^2}\sum_{i=1}^d \lambda_i + 1} = d\log\brk*{\frac{\eta_t^2V_t}{d\sigma_t^2} + 1},
    \end{equation*}
    where the only inequality follows by the fact that the geometric mean is always less than the arithmetic mean. Let $V_t^i = \tr[\V_t^i]$, then through the same strategy, one can prove that for all $i \in \{1, \cdots, r\}$:
    \begin{equation*}
        \theta_c(\V_t^i) \le \theta_v(V_t^i),
    \end{equation*}
    then by Jensen's inequality, we have:
    \begin{equation*}
        \sum_{i=1}^r \theta_c(\V_t^i) \le \sum_{i=1}^r \theta_v(V_t^i) = \sum_{i=1}^r d\log\brk*{\frac{\eta_t^2V_t^i}{d\sigma_t^2} + 1} \le d\log\brk*{\frac{\eta_t^2}{d\sigma_t^2} \sum_{i=1}^r V_t^i  + 1} = d\log\brk*{\frac{\eta_t^2}{d\sigma_t^2} V_t  + 1} = \theta_v(V_t).
    \end{equation*}
    Next, by applying Lemma \ref{lm:det} recursively, we can prove that
    \begin{equation*}
        \theta_c(\V_t) = \frac{1}{2}\log\abs*{\frac{\eta_t^2}{\sigma_t^2}\V_t+I} \le \frac{1}{2}\log\prod_{i=1}^r\abs*{\frac{\eta_t^2}{\sigma_t^2}\V_t^i+I} = \frac{1}{2}\sum_{i=1}^r \log\abs*{\frac{\eta_t^2}{\sigma_t^2}\V_t^i+I} = \sum_{i=1}^r \theta_c(\V_t^i).
    \end{equation*}
    To prove the last inequality of Proposition \ref{prop:tight}, notice that
    \begin{align*}
        V_t &= \E_{B_t}[\norm{g(W_{t-1},B_t) - \E_{B_t}[g(W_{t-1},B_t)]}_2^2] = \E_{B_t}[\norm{g(W_{t-1},B_t)}_2^2] - \norm{\E_{B_t}[g(W_{t-1},B_t)]}_2^2 \\
        &\le \E_{B_t}[\norm{g(W_{t-1},B_t)}_2^2] \le \max_{w \in \mathcal{W}, z \in \mathcal{Z}}\norm{g(w,z)}_2^2 = L,
    \end{align*}
    which finishes the proof by the monotonicity of the log function.
\end{proof}

\begin{table}[t]
    \centering
    \caption{Hyper-parameters used to train deep learning models.}
    \label{tbl:hparam}
    \begin{tabular}{ l|ccc }
        \hline
        Hyper-parameter & Synthetic & MNIST & CIFAR10 \\
        \hline
        learning rate ($\eta$) & 0.001 & 0.01 & 0.01 \\
        size of training dataset ($n$) & 100 & 5000 & 5000 \\
        epochs & 50 & 100 & 100 \\
        batch size & 10 & 50 & 50 \\
        steps ($T$) & 500 & 10000 & 10000 \\
        variance of the noise ($\sigma_t^2$) & $10^{-3}$ & $10^{-5}$ & $10^{-5}$ \\
        \hline
    \end{tabular}
\end{table}
\begin{figure}[t]
    \centering
    \includegraphics[width=0.55\textwidth]{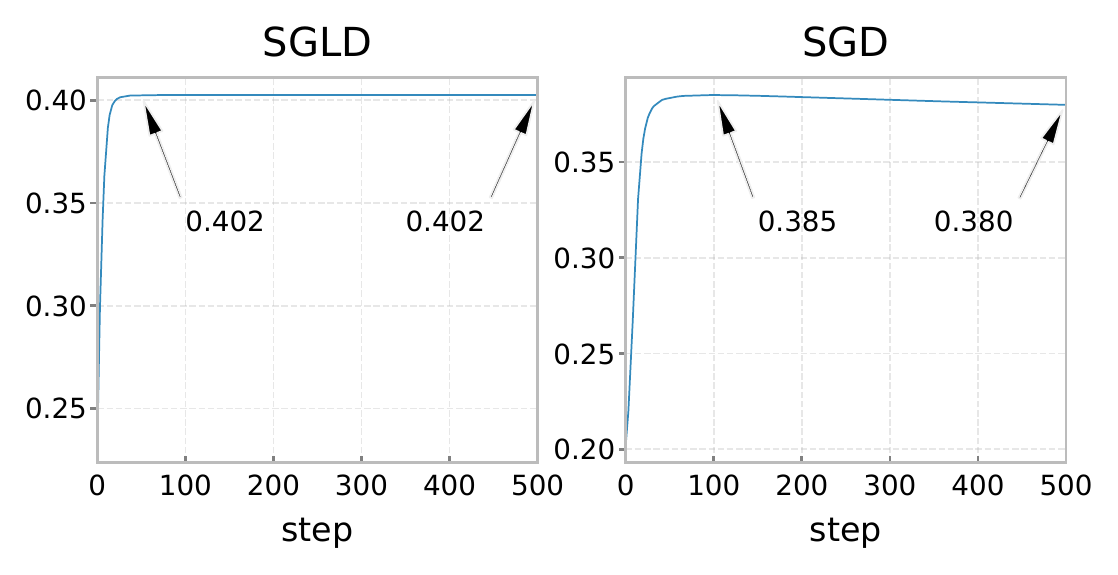}
    \caption{Behavior of $I_1(W_T;S)$ on synthetic data.}
    \label{fig:IWS}
\end{figure}

\subsection{Proof of Theorem \ref{thm:gen_sgd}}
\begin{lemma} \label{lm:mi_bound2}
    The mutual information $I_1(\tilde{W}_T;S) \le \sum_{t=1}^T I_1(-\eta_t g(W_{t-1},B_t) + \tilde{\xi}_t;B_t|\tilde{W}_{t-1})$.
\end{lemma}
\begin{proof}
    \begin{align}
        I_1(\tilde{W}_T;S) &= I_1(\tilde{W}_{T-1} - \eta_T g(W_{T-1}, B_T) + \tilde{\xi}_T; S) \nonumber \\
        &\le I_1(\tilde{W}_{T-1}, - \eta_T g(W_{T-1}, B_T) + \tilde{\xi}_T; S) \label{eq:mi_bound4} \\
        &= I_1(\tilde{W}_{T-1}; S) + I_1(- \eta_T g(W_{T-1}, B_T) + \tilde{\xi}_T; S|\tilde{W}_{T-1}) \label{eq:mi_bound5} \\
        &\le I_1(\tilde{W}_{T-2}; S) + I_1(- \eta_{T-1} g(W_{T-2}, B_{T-1}) + \tilde{\xi}_{T-1}; S|\tilde{W}_{T-2}) \nonumber \\
        &\qquad+ I_1(- \eta_T g(W_{T-1}, B_T) + \tilde{\xi}_T; S|\tilde{W}_{T-1}) \nonumber \\
        &\le \cdots \nonumber \\
        &\le I_1(\tilde{W}_0;S) + \sum_{t=1}^T I_1(-\eta_t g(W_{t-1}, B_t) + \tilde{\xi}_t; S|\tilde{W}_{t-1}) \label{eq:mi_bound6} \\
        &= \sum_{t=1}^T I_1(-\eta_t g(W_{t-1}, B_t) + \tilde{\xi}_t; S|\tilde{W}_{t-1}) \label{eq:mi_bound7} \\
        &\le \sum_{t=1}^T I_1(-\eta_t g(W_{t-1}, B_t) + \tilde{\xi}_t; B_t|\tilde{W}_{t-1}) \label{eq:mi_bound8} \\
        &= \sum_{t=1}^T I_1(\tilde{W}_t - \tilde{W}_{t-1}; B_t|\tilde{W}_{t-1}) = \sum_{t=1}^T I_1(\tilde{W}_t; B_t|\tilde{W}_{t-1}), \nonumber
    \end{align}
    where (\ref{eq:mi_bound4}) follows by noticing that $Z \rightarrow (X,Y) \rightarrow f(X,Y)$ forms a Markov chain and then apply property \ref{prop:data_proc} in Proposition \ref{prop:property}, (\ref{eq:mi_bound5}) follows by property \ref{prop:cond_mi} in Proposition \ref{prop:property}, (\ref{eq:mi_bound6}) follows by repeating the steps above (\ref{eq:mi_bound5}), (\ref{eq:mi_bound7}) follows by noticing that $\tilde{W}_0$ and $S$ are independent, and (\ref{eq:mi_bound8}) follows by the fact that $S \rightarrow B_t \rightarrow -\eta_t g(w, B_t) + \tilde{\xi}_t|w = \tilde{W}_{t-1}$ form a Markov chain.
\end{proof}

\begin{restatetheorem}{\ref{thm:gen_sgd}}[Restate]
    Under the same conditions of Lemma \ref{lm:gen_sgd}:
    \begin{gather*}
        I_1\prn*{\tilde{W}_T;S} \le \sum_{t=1}^T\prn*{\frac{1}{2}\log\abs*{\frac{\eta_t^2}{\sigma_t^2}\V_t+I} + E_{\tilde{W}_t|\tilde{W}_{t-1}}^\kappa}, \\
        I\prn*{\tilde{W}_T;S} \le \sum_{t=1}^T \frac{1}{2}\log\abs*{\frac{\eta_t^2}{\sigma_t^2}\V_t+I}.
    \end{gather*}
\end{restatetheorem}
\begin{proof}
    Applying Lemma \ref{lm:mi_bound2}, we have
    \begin{equation*}
        I_1(\tilde{W}_T;S) \le \sum_{t=1}^T\prn*{\frac{1}{2}\log\abs*{\frac{\eta_t^2}{\sigma_t^2}\Cov[g(W_{t-1},B_t)]+I} + E_{\tilde{W}_t|\tilde{W}_{t-1}}^\kappa},
    \end{equation*}
    by applying Lemma \ref{lm:mi_bound1} with $X = B_t$, $Y = \tilde{W}_{t-1}$, $\Delta = -\Delta_{t-1}$, $\xi = \tilde{\xi}_t$ and $f(W_{t-1}, B_t) = -\eta_t g(W_{t-1}, B_t)$.
\end{proof}

\begin{figure}[t]
    \centering
    \includegraphics[width=0.55\textwidth]{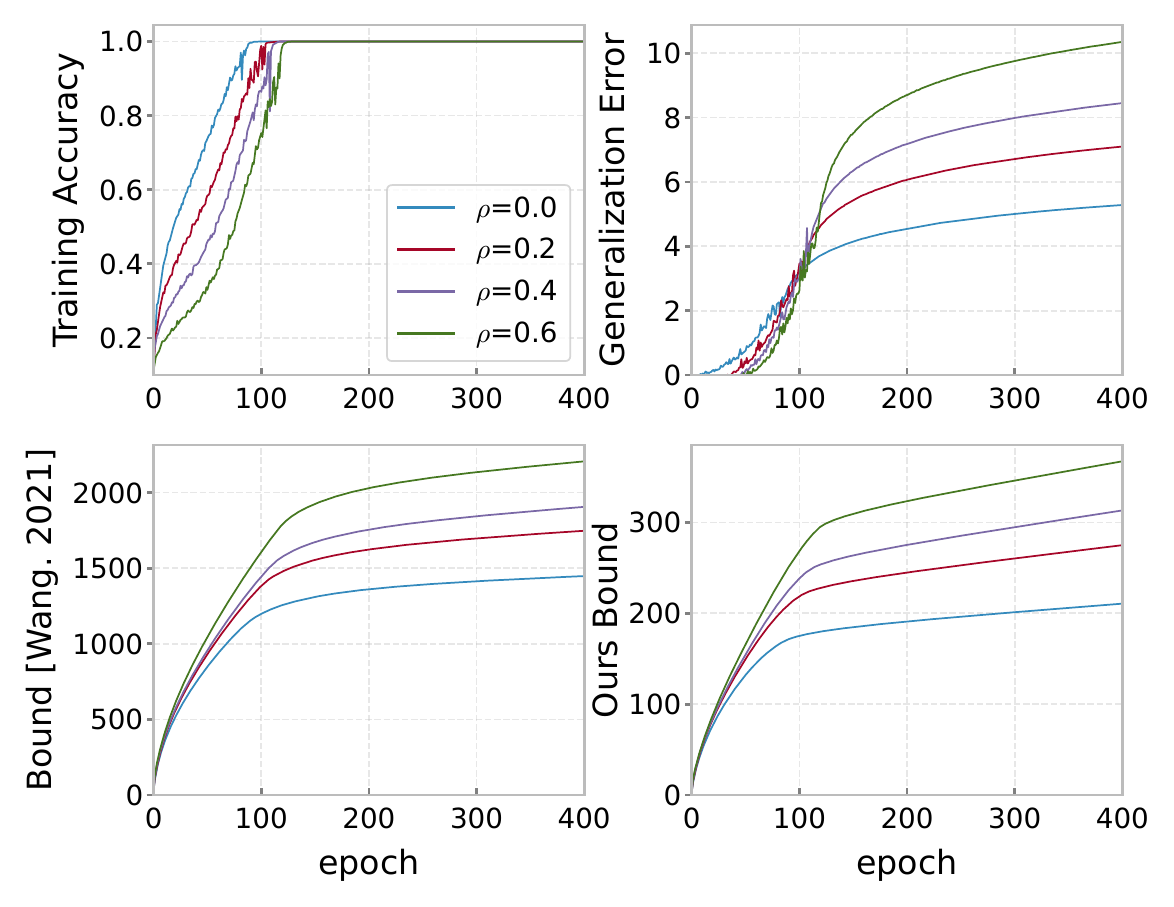}
    \caption{Random label experiment on CIFAR10.}
    \label{fig:label_cifar}
\end{figure}
\begin{figure}[t]
    \centering
    \includegraphics[width=0.55\textwidth]{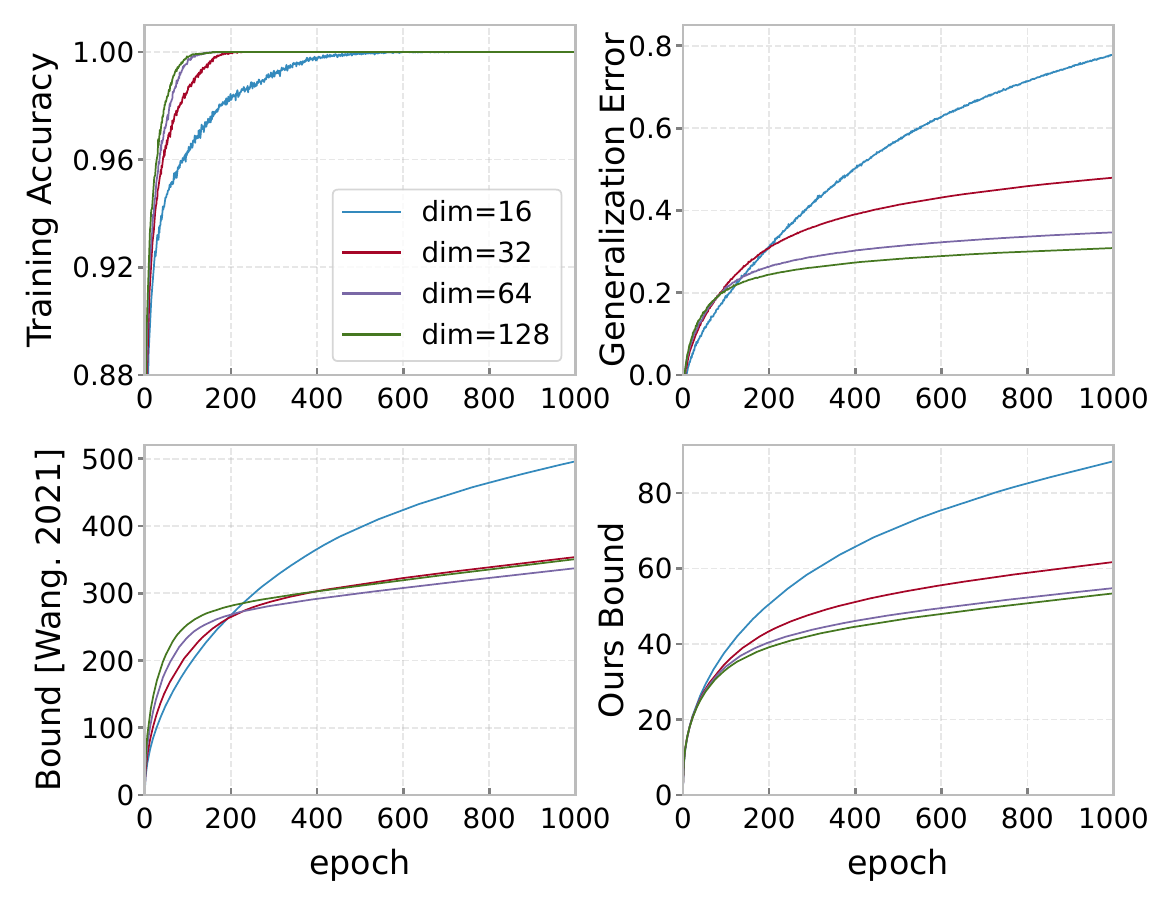}
    \caption{Generalization of MLP on MNIST with different dimensionality of the hidden layer.}
    \label{fig:width_mnist}
\end{figure}

\section{Experiment Details} \label{sec:supp_expr}

Deep learning models are trained with an Intel Xeon CPU (2.10GHz, 48 cores), 256GB memory, and 4 Nvidia Tesla V100 GPU (32GB). For the MNIST data set, we train an MLP with one hidden layer of size 128. For the CIFAR10 dataset, we train a CNN with 4 convolution layers (32, 32, 48, 48) of size 3 $\times$ 3 followed by two fully connected layers of size 48. All of the layers above use ReLU as the activation function. The hyper-parameters used for training are listed in Table \ref{tbl:hparam}.

As can be seen in Figure \ref{fig:IWS}, the curves of IWS soon stop increasing after several epochs and start to decrease. This is consistent with the behavior of the true generalization error, supporting our claim that the estimate of $I_1(W_T;S)$ successfully reflects the behavior of $I(W_T;S)$.

In figure \ref{fig:label_cifar}, we conduct the random label experiment on CIFAR10, and in Figure \ref{fig:width_mnist}, we test the generalization behavior of MLP on the MNIST dataset with different dimensionality of the hidden layer. It can be seen that our bound is still $5$ more times tighter than the bound of Lemma \ref{lm:gen_sgd}.

\end{document}